\documentclass[twoside]{article}

\usepackage[accepted]{aistats2024}

\usepackage[round]{natbib}

\usepackage{amsmath}
\usepackage{amssymb}
\usepackage{mathtools}
\usepackage{amsthm}
\usepackage{amsfonts}
\usepackage{bm}
\usepackage{color}
\usepackage{comment}
\usepackage{thmtools}
\usepackage{thm-restate}
\usepackage{caption}
\usepackage{subcaption}
\usepackage{xspace}
\usepackage{stmaryrd}
\usepackage{dsfont}
\usepackage{enumitem}

\usepackage{algorithmnames}
\usepackage[utf8]{inputenc} 
\usepackage[T1]{fontenc}    
\usepackage[hidelinks]{hyperref}       
\usepackage{url}            
\usepackage{booktabs}       
\usepackage{nicefrac}       
\usepackage{microtype}      
\usepackage{xcolor}         

\declaretheorem[name=Theorem]{thr}

\declaretheorem[name=Assumption]{ass}

\declaretheorem[name=Lemma,sharenumber=thr]{lemma}

\declaretheorem[name=Remark]{remark}

\usepackage[capitalize,noabbrev]{cleveref}
 
\newcommand{\inner}[2]{\langle #1, #2 \rangle}
\newcommand{\Vs}{\mathbf{V}}
\newcommand{\Is}{\mathbf{I}}
\newcommand{\bs}{\mathbf{b}}
\newcommand{\vz}{\mathbf{z}}
\newcommand{\zs}{\mathbf{z}}
\newcommand{\As}{\mathcal{A}}
\newcommand{\vgamma}{\bm{\gamma}}

\usepackage{tikz}
\usepackage{tikzscale}
\usepackage{pgfplots}
\usetikzlibrary{matrix}
\usepgfplotslibrary{groupplots}
\pgfplotsset{compat=newest}

\usepackage{algorithm2e}

\usepackage[textsize=tiny, colorinlistoftodos]{todonotes}

\DeclareRobustCommand{\eg}{e.g.,\@\xspace}                         
\DeclareRobustCommand{\ie}{i.e.,\@\xspace}                         
\DeclareRobustCommand{\wrt}{w.r.t.\@\xspace}

\DeclareRobustCommand{\quotes}[1]{``#1''}

\DeclareMathOperator*{\argmax}{arg\,max}
\DeclareMathOperator*{\argmin}{arg\,min}

\newcommand{\dsb}[1]{\llbracket #1 \rrbracket}

\makeatletter
\newcommand{\namelabel}[2]{%
  \phantomsection
  \renewcommand{\@currentlabel}{#1}
  \label{#2}
}
\makeatother

\begin{document}

\runningauthor{Bacchiocchi, Genalti, Maran, Mussi, Restelli, Gatti and Metelli}

\twocolumn[

\aistatstitle{Autoregressive Bandits}

\aistatsauthor{Francesco Bacchiocchi{\normalfont *}\textsuperscript{{\normalfont ,1}} \ \ \ Gianmarco Genalti{\normalfont *}\textsuperscript{{\normalfont ,1}} \ \ \  Davide Maran{\normalfont *}\textsuperscript{{\normalfont ,1}} \ \ \ Marco Mussi{\normalfont *}\textsuperscript{{\normalfont ,1}} \\ \textbf{Marcello Restelli\textsuperscript{{\normalfont 1}} \ \ \ Nicola Gatti\textsuperscript{{\normalfont 1}} \ \ \ Alberto Maria Metelli\textsuperscript{{\normalfont 1}}}}
\aistatsaddress{\textsuperscript{1}Politecnico di Milano \\ *Equal Contribution} ]

\allowdisplaybreaks[4]

\begin{abstract}
Autoregressive processes naturally arise in a large variety of real-world scenarios, including stock markets, sales forecasting, weather prediction, advertising, and pricing. When facing a sequential decision-making problem in such a context, the temporal dependence between consecutive observations should be properly accounted for guaranteeing convergence to the optimal policy. In this work, we propose a novel online learning setting, namely, Autoregressive Bandits (ARBs), in which the observed reward is governed by an autoregressive process of order $k$, whose parameters depend on the chosen action. We show that, under mild assumptions on the reward process, the optimal policy can be conveniently computed. Then, we devise a new optimistic regret minimization algorithm, namely, \algname (\algnameshort), that suffers sublinear regret of order $\widetilde{\mathcal{O}} \left( \frac{(k+1)^{3/2}\sqrt{nT}}{(1-\Gamma)^2}\right)$, where $T$ is the optimization horizon, $n$ is the number of actions, and $\Gamma < 1$ is a stability index of the process. Finally, we empirically validate our algorithm, illustrating its advantages w.r.t.~bandit baselines and its robustness to misspecification of key parameters.
\end{abstract}

\section{INTRODUCTION}
\label{sec:intro}

In a large variety of sequential decision-making problems, a learner is required to choose an action that, when executed, determines: ($i$) the immediate reward and ($ii$) the behavior of an underlying process that will influence, in some unknown manner, the future rewards. This process is influenced by the course of actions the agent performs and generates a temporal dependence between the sequence of observed rewards. A class of stochastic processes widely employed to model the temporal dependencies in real-world phenomena are the \emph{autoregressive} (AR) processes~\citep{hamilton2020time}. In this paper, we model the reward of a sequential decision-making problem as an AR process whose parameters depend on the action selected by the agent at every round. This scenario can be represented as a particular class of continuous \textit{reinforcement learning} problems~\citep{sutton2018reinforcement} where an AR process governs the temporal structure of the observed rewards through the action-dependent AR parameters that are unknown to the agent. It is worth mentioning that such a scenario displays notable differences compared to more traditional \emph{non-stationary} learning problems. Indeed, in the scenario we address, the environment does not change, and the reward dynamics depend on the agent's course of actions only. 

\textbf{Original Contribution}~~~In this work we propose a novel setting, named \emph{\settingname} (\settingnameshort), in which the reward follows an AR process of order $k$ whose parameters depend on the agent's actions. Importantly, we show that the optimal policy, differently from many bandit models, is \emph{stationary} and \emph{closed-loop}, as the optimal action depends on the previously observed rewards (Section~\ref{sec:setting}). Then, we devise a new optimistic algorithm, namely \algname (\algnameshort), to learn an optimal policy in an online fashion (Section~\ref{sec:algorithm}), and we show that it suffers sublinear regret of order $\widetilde{\mathcal{O}} \left( \frac{(k+1)^{3/2}\sqrt{nT}}{(1-\Gamma)^2}\right)$, where $T$ is the optimization horizon, $n$ is the number of actions, and $\Gamma < 1$ is a stability index of the process (Section~\ref{sec:regret}). Finally, we empirically evaluate \algnameshort comparing its performance with several bandit baselines with competitive results and illustrating its notable robustness w.r.t.~the misspecification of key parameters (Section~\ref{sec:experiments}). 
\section{PROBLEM FORMULATION}
\label{sec:setting}
In this section, we introduce the \settingnameshort setting, formalize the learning problem, how the learner interacts with the environment, assumptions, policies and definition of regret (Section~\ref{subsec:setting}). Subsequently, we derive a closed-form solution for the optimal policy of an \settingnameshort (Section~\ref{sec:optimalPolicy}).

\subsection{Setting}\label{subsec:setting}
We study the sequential interaction between a learner and an environment. At every round $t\in\mathbb{N}$, the learner chooses an action $a_t \in \mathcal{A} \coloneqq \dsb{n}$, among the $n \in \mathbb{N}$ available ones.\footnote{Along the paper, we use the following notation. Let $a,b\in\mathbb{N}$ with $a \le b$, we denote with $\dsb{a,b} := \left\{a,\ldots,b\right\}$, and with $\dsb{b}:=\left\{1,\ldots,b\right\}$. Let $\mathbf{x},\mathbf{y} \in \mathbb{R}^n$ be real-valued vectors, we denote with $\inner{\mathbf{x}}{\mathbf{y}}=\mathbf{x}^T \mathbf{y}=\sum_{i=1}^n x_i y_i$ the inner product. For a positive semidefinite matrix $\mathbf{A}\in\mathbb{R}^{n\times n}$, we denote with $\left\|\mathbf{x} \right\|_{\mathbf{A}}^2 = \mathbf{x}^T \mathbf{A} \mathbf{x}$ the weighted 2-norm. A zero-mean random variable $\xi$ is $\sigma^2$-subgaussian if $\mathbb{E}[e^{\lambda \xi}] \le e^{\frac{\lambda^2\sigma^2}{2}}$, for every $\lambda \in \mathbb{R}$.} In the \settingnameshort setting, the reward evolves according to an \textit{autoregressive process} of order $k$~\citep[AR($k$),][]{hamilton2020time}. Thus, the learner observes a noisy reward $x_t$ of the form:
\begin{align}
x_t = \gamma_0(a_t) + \sum_{i=1}^{k}\gamma_i(a_t) x_{t-i}+\xi_t,
\end{align}
where $x_t \in \mathcal{X}$ ($\mathcal{X}\subseteq \mathbb{R}$ is the reward space), $\gamma_0(a_t)\in\mathbb{R}$ and $(\gamma_i(a_t))_{i\in\dsb{k}}\in\mathbb{R}^k$ are the unknown \emph{parameters} depending on chosen action $a_t$, and $\xi_t$ is a zero-mean $\sigma^2$-subgaussian random noise, independent conditioned to the past.
The reward evolution can be expressed in an alternative form as follows:\footnote{Although the linear structure might resemble the \emph{contextual linear bandits}~\citep{ChuLRS11}, the two settings are non-comparable. Indeed, in our ARBs the vector $\zs_{t-1}$ is not sampled independently at every round, but, instead, follows a sequential process depending on the past, making the decision problem way more challenging.}
\begin{align}
   x_t =& \inner{\vgamma(a_t)}{\zs_{t-1}}+\xi_t, 
\end{align}
where $\zs_{t-1}:=(1, x_{t-1}, \ldots, x_{t-k})^T \in \mathcal{Z} \coloneqq \{1\} \times \mathcal{X}^{k}$ is the \textit{vector of past rewards} expressing past history, and $\vgamma(a):=(\gamma_0(a),\ldots,\gamma_k(a))^T \in \mathbb{R}^{k+1}$ is the \textit{parameter vector}, defined for all the actions $a\in\mathcal{A}$. It is worth noting that when $\gamma_i(a) = 0$ for all $i \in \dsb{k}$ and $a \in \mathcal{A}$, the ARB setting reduces to a standard MAB~\citep{auer2002finite}.

\textbf{Assumptions}~~~We introduce the assumption that we employ in the paper and comment on its role.

\begin{ass}\label{ass:all}
The parameters $(\gamma_i(a))_{i \in \dsb{0,k}}$ fulfill the following conditions:
\begin{enumerate}[topsep=-1pt, noitemsep, label=\alph*.]
	\item (Non-negative coefficients) $\gamma_i(a) \ge 0 $ for every $a \in \As$, $i \in \dsb{0,k}$; \namelabel{1.a}{ass:monotonicity}
	\item (Stability) $\Gamma \coloneqq \max_{a \in \As} \sum_{i=1}^k \gamma_i(a)  < 1$; \namelabel{1.b}{ass:stability}
	\item (Boundedness) $m \coloneqq \max_{a \in \As} \gamma_0(a) < +\infty$. \namelabel{1.c}{ass:boundedness}
\end{enumerate}
\end{ass}
First, Assumption~\ref{ass:monotonicity} requires that the coefficients of the AR process are non-negative. This scenario is ubiquitous in real-world AR phenomena (\eg pricing, stock markets, digital advertising), where processes violating such an assumption will generate unrealistic sign alternation behaviours. An extensive discussion and a graphical elaboration about this assumption are provided in Appendix~\ref{apx:assumption1a}. Assumption~\ref{ass:stability} requires that the sum of $(\gamma_i(a))_{i \in \dsb{k}}$ is limited to a value $\Gamma \in [0, 1)$ and Assumption~\ref{ass:boundedness} enforces the boundedness of $\gamma_0(a)$. These latter assumptions guarantee that the AR process does not diverge in expectation regardless of the sequence of the actions played.

\textbf{Policies and Regret}~~~The learner's behavior is modeled by a deterministic policy $\boldsymbol\pi = (\pi_t)_{t \in \mathbb{N}}$ defined, for every round $t\in\mathbb{N}$ as $\pi_t : \mathcal{H}_{t-1} \rightarrow \mathcal{A}$, mapping the history of observations $H_{t-1} = (x_0,a_1,x_1,\ldots, a_{t-1},x_{t-1})\in\mathcal{H}_{t-1}$ to an action $a_t=\pi_t(H_{t-1})\in\mathcal{A}$ where $\mathcal{H}_{t-1} = \mathcal{X} \times (\mathcal{A} \times \mathcal{X})^{t-1}$ is the set of histories of length $t-1$. The performance of a policy $\boldsymbol\pi$ is evaluated in terms of the \textit{expected cumulative reward} over the horizon $T \in \mathbb{N}$, defined as:
\begin{equation}\label{eq:reward}
    J_T(\boldsymbol\pi):=\mathbb{E}\left[\sum_{t=1}^T x_t\right]
\end{equation}
with:
\begin{align*}
      x_t &= \inner{\vgamma(a_t)}{\zs_{t-1}} + \xi_t , \\
      a_t &= \pi_t(H_{t-1}),    
\end{align*}
where the expectation is taken w.r.t.~the randomness of the reward noise $\xi_t$. A policy $\boldsymbol\pi^*$ is \textit{optimal} if it maximizes the expected average reward, \ie $\boldsymbol\pi^*\in \argmax_{\boldsymbol\pi} J_T(\boldsymbol\pi)$, whose performance is denoted as $J_T^*:=J_T(\boldsymbol\pi^*)$. The goal of the learner is to minimize the \textit{expected cumulative (policy) regret} by playing a policy $\boldsymbol \pi$, competing against the optimal policy $\boldsymbol\pi^*$ over a \textit{learning horizon} $T\in\mathbb{N}^+$:
\begin{align}\label{eq:regretDef}
    R(\boldsymbol\pi,T)= J_T^* - J_T(\boldsymbol \pi) = \mathbb{E} \left[ \sum_{t=1}^T r_t \right],
\end{align}
where $r_t \coloneqq x_t^* - x_t$ is the instantaneous policy regret and $(x_t^*)_{t \in \dsb{T}}$ is the sequence of rewards observed by playing the optimal policy $\bm{\pi}^*$.

\subsection{Optimal Policy}\label{sec:optimalPolicy}
In this section, we derive a closed-form expression for the optimal policy $\boldsymbol\pi^*$ for the expected cumulative reward of  Equation~\eqref{eq:reward}, under Assumption~\ref{ass:monotonicity}.

\begin{restatable}[Optimal Policy]{thr}{OptimalPolicy}
\label{thr:optimal_policy}
Under Assumption~\ref{ass:monotonicity}, for every round $t \in \mathbb{N}$, the optimal policy $\pi_t^*(H_{t-1})$ satisfies:
\begin{equation}
\label{eq:opt_policy}
    \pi_t^*(H_{t-1}) \in \argmax_{a \in \mathcal{A}} \; \inner{\vgamma(a)}{\zs_{t-1}}.
\end{equation}
\end{restatable}

This result deserves some comments. First, the optimal action depends on the vector of past rewards $\zs_{t-1}$ and, thus, on the most recent $k$ rewards $x_{t-1},\dots,x_{t-k}$ only. Thus, the optimal policy $\bm{\pi}^*$ is non-Markovian with memory $k$ or, equivalently, Markovian w.r.t.~the state representation $\zs_{t-1}$.\footnote{We can look at the ARB as a particular \emph{Markov Decision Processes}~\citep[MDPs,][]{puterman2014markov} with $\zs_{t-1} \in \mathcal{Z}$ as state representation.} Second, the optimal action maximizes, at every round $t \in \mathbb{N}$, the \emph{expected instantaneous reward} $\mathbb{E}[x_t|H_{t-1}] =  \inner{\vgamma(a)}{\zs_{t-1}}$. This is a consequence of the non-negativity of the parameters $\gamma_i(a)$ (Assumption~\ref{ass:monotonicity}), which enforces a meaningful evolution of the AR process, compatible with our real-world motivating scenarios. This way, the action maximizing the expected \emph{immediate} reward (\ie a \emph{myopic} policy) is optimal for the expected \emph{cumulative} reward too. 
The proof can be found in Appendix~\ref{apx:omitted_proofs}.
\section{AUTOREGRESSIVE UPPER CONFIDENCE BOUND}
\label{sec:algorithm}

In this section, we present \algname (\algnameshort), an optimistic regret minimization algorithm for the \settingnameshort setting whose pseudo-code is reported in Algorithm~\ref{alg:alg}. \algnameshort leverages the myopic optimal policy for ARBs (Theorem~\ref{thr:optimal_policy}) and implements an incremental regularized least squares procedure to estimate the unknown parameters $\vgamma(a)$, for every action $a \in \mathcal{A}$ independently. The algorithm requires the knowledge of the order $k$ of the AR process, although this knowledge can be replaced with the one of an upper bound $\overline{k} > k$ of the AR order.\footnote{Indeed, any AR process of order $k$ can be regarded as an AR process of order $\overline{k} > k$ setting $\gamma_i(a) = 0$ for $i \in \dsb{k+1,\overline{k}}$. An empirical validation of the \algnameshort performances in the case of a misspecified $k$ is provided in Section~\ref{subsec:k_misspec}.} 

\algnameshort starts by initializing for all the actions $a \in \mathcal{A}$ the Gram matrix $\Vs_0(a) = \lambda \Is_{k+1}$, where $\lambda >0$ is the Ridge regularization parameter, the vectors $\bs_0(a)=\widehat{\vgamma}_0(a)= \mathbf{0}_{k+1}$, and the observations vector $\mathbf{z}_0 = (1,0,\dots,0)^T$ (line~\ref{line:init}).\footnote{We assume to know the initial observations vector $\mathbf{z}_0$. If this is not the case, we can play an arbitrary action for the first $k$ rounds to observe $(x_t)_{t\in\dsb{k}}$ with just an additional constant loss term.}
Then, for each round $t \in \dsb{T}$, \algnameshort computes the \textit{Upper Confidence Bound} (UCB) index (line~\ref{line:ucb}) for every $a \in \mathcal{A}$. Such an optimistic index is composed of the inner product between the estimated value of $\vgamma (a)$ and the state representation $\vz_{t-1}$, plus the confidence interval $\beta_{t-1}(a)$. Formally:
\begin{align}
    {a}_t \in & \argmax_{a \in \mathcal{A}} \text{ UCB}_t(a) \coloneqq \nonumber \\
    & \inner{\widehat{\vgamma}_{t-1}(a)}{\vz_{t-1} } + \beta_{t-1}(a) \left\|\vz_{t-1}(a) \right\|_{\Vs_{t-1}(a)^{-1}}, \label{eq:ucbEq}
\end{align}
where $\widehat{\vgamma}_{t-1}(a)$ is the most recent estimate of the parameter vector $\vgamma(a)$, $\vz_{t-1} =  (1,x_{t-1},\dots,x_{t-k})^T$ is the observations vector, and $\beta_{t-1}(a) \ge 0$ is an exploration coefficient that will be defined later (Section~\ref{sec:regret}). 
The index $\text{UCB}_t(a)$ is designed to be optimistic, \ie $\inner{\vgamma(a)}{\vz_{t-1}} \le \text{UCB}_t(a)$ with high probability for all $a \in \mathcal{A}$. 
Then, action ${a}_t$ is executed (line~\ref{line:play}) and the new reward $x_t$ is observed. 
This sample is employed to update the Gram matrix estimate $\Vs_t({a}_t)$, the vector $\bs_t({a}_t)$, and the estimate $\widehat{\vgamma}_t({a}_t)$ (line~\ref{line:updategram}).

\RestyleAlgo{ruled}
\LinesNumbered
\begin{algorithm}[t]
\caption{\algnameshort.}\label{alg:alg}
{\small
\SetKwInOut{Input}{Input}
\Input{Regularization parameter $\lambda > 0$, autoregressive order $k$, exploration coefficients $(\beta_{t-1})_{t \in \dsb{T}}$}
Initialize $t \leftarrow 1$, $\Vs_0(a) = \lambda \Is_{k+1}$, $\bs_0(a) = \mathbf{0}_{k+1}$, $\widehat{\vgamma}_0(a) = \mathbf{0}_{k+1}$, $\forall a \in \mathcal{A}$, $\zs_0 = (1,0,\dots,0)^T$\label{line:init}

\For{$t \in \dsb{T}$ \label{line:3}}{
    Compute $ {a}_t \in \argmax_{a \in \mathcal{A}} \text{UCB}_t(a) \coloneqq \inner{\widehat{\vgamma}_{t-1}(a)}{\vz_{t-1} } + \beta_{t-1}(a) \left\| \vz_{t-1} \right\|_{\Vs_{t-1}(a)^{-1}} $ \label{line:ucb}
    
    Play action ${a}_t$ and observe $x_{t} = \inner{\vgamma({a}_t)}{\vz_{t-1}} + \xi_t$ \label{line:play}

    Update for all $a \in \mathcal{A}$:
    
$\;\;\begin{array}{l}
    \Vs_{t}(a) = \Vs_{t-1}(a) + \zs_{t-1}\zs_{t-1}^T \mathds{1}_{\{a={a}_t\}}\\
    \bs_{t}(a) = \bs_{t-1}(a) + \zs_{t-1} x_t \mathds{1}_{\{a={a}_t\}} \\
    \widehat{\vgamma}_t(a) = \Vs_{t}(a)^{-1} \bs_{t}(a) \label{line:updategamma}
    \end{array}
    $\label{line:updategram}
    
    Update $\vz_t =  (1,x_t,\dots,x_{t-k+1})^T$
    
    $t \leftarrow t+1$
}
}
\end{algorithm}
\section{REGRET ANALYSIS}
\label{sec:regret}

In this section, we present the analysis of the regret of \algnameshort. We start providing a self-normalized concentration inequality for estimating the AR parameters $\vgamma(a)$ (Section~\ref{sec:concentration}). Then, we derive a decomposition of the regret (Section~\ref{sec:regretAnalysis}) that is useful to complete the analysis and, finally, we present the bound on the expected cumulative (policy) regret (Section~\ref{sec:boundRegret}). The complete proofs of the theorems stated in this section can be found in Appendix~\ref{apx:omitted_proofs}.

\subsection{Concentration Inequality for the Parameter Vectors}\label{sec:concentration}
We start by providing a concentration result for the estimates $\widehat\vgamma_t(a)$ of the true parameter vector $\vgamma(a)$, for every action $a\in\mathcal{A}$, as performed in Algorithm~\ref{alg:alg}.
At the end of each round $t \in \mathbb{N}$, where the chosen action is $a_t\in\mathcal{A}$, we solve the Ridge-regularized linear regression problem and update the coefficient vector estimate $\widehat\vgamma_t(a_t)$ associated to $a_t$:
\begin{align*}
 \widehat\vgamma_t(a_t) & = \argmin_{\widetilde\vgamma\in\mathbb{R}^{k+1}} \sum_{l\in\mathcal{O}_t(a_t)} (x_l - \inner{\widetilde\vgamma}{\zs_{l-1}})^2 + \lambda \left\|\widetilde\vgamma\right\|_2^2 \\
 & = \Vs_t(a_t)^{-1}\bs_t(a_t),
\end{align*}
where $\mathcal{O}_t(a)$ is the set of rounds where action $a$ has been chosen, \ie $\mathcal{O}_t(a)\coloneqq \left\{\tau \in \dsb{t}: a_\tau = a\right\}$.
The following result shows how the estimate $\widehat\vgamma(a)$ concentrates around the true parameters $\vgamma(a)$ over the rounds.

\begin{restatable}[Self-Normalized Concentration]{lemma}{Concentration}
\label{thr:concentration}
Let $a \in \As$ be an action, let $(\widehat\vgamma_t(a))_{t\in\mathcal{O}_\infty(a)}$ be the sequence of solutions to the Ridge regression problems computed by Algorithm~\ref{alg:alg}. Then, for every regularization parameter $\lambda > 0$, confidence $\delta \in (0,1)$, simultaneously for every round $t \in \mathbb{N}$ and action $a \in \mathcal{A}$, with probability at least $1-\delta$ it holds that:
\begin{align*}
    & \left\| \widehat{\vgamma}_t(a) - \vgamma(a) \right\|_{\Vs_t(a)} \le \\
    & \qquad \sqrt{\lambda} \| \vgamma(a) \|_2 + \sigma \sqrt{2 \log\left( \frac{n}{\delta} \right) + \log \left( \frac{\det \Vs_t(a)}{\lambda^{k+1}} \right) }.
\end{align*}
\end{restatable} 

Lemma~\ref{thr:concentration} resembles the self-normalized concentration inequality of~\citep[][Theorem 1]{abbasi2011improved}. However, contrary to \linucb~\citep{abbasi2011improved}, the exploration coefficients $\beta_t(a)$ are different for every action $a\in\mathcal{A}$. Lemma~\ref{thr:concentration} allows properly defining the exploration coefficients $\beta_{t}(a)$ employed in Algorithm~\ref{alg:alg}, defined for every action $a \in \As$ and round $t \in \dsb{0,T-1}$:
\begin{align}
\beta_{t}(a) \coloneqq & \sqrt{\lambda (m^2+1)} + \nonumber \\ &  + \sigma  \sqrt{2 \log\left( \frac{n}{\delta} \right) + \log \left( \frac{\det \Vs_t(a)}{\lambda^{k+1}} \right) }. \label{eq:beta}
\end{align}
This formula contains two terms. The first one is a \emph{bias} term that increases with $m$ (\ie the maximum value of the largest $\gamma_0(a)$ over the actions $a \in \mathcal{A}$, see Assumption~\ref{ass:boundedness})  and with the regularization parameter of the Ridge regression $\lambda > 0$. The second one is the \emph{concentration} term and increases with the subgaussian parameter $\sigma$ of the noise, the number of actions $n$, and the determinant of the design matrix $\Vs_t(a)$, but decreases in $\lambda$. It is worth noting that $\beta_{t}(a)$ is obtained from Lemma~\ref{thr:concentration}, by observing that, under Assumptions~\ref{ass:stability} and~\ref{ass:boundedness}, we have $ \| \vgamma(a) \|_2 \le \sqrt{m^2 + \Gamma^2} \le \sqrt{m^2+1}$. Thus, the exploration coefficient $\beta_t(a)$ ensures that, with probability $1 - \delta$, the following inequality holds simultaneously for all actions $ a \in \mathcal{A}$ and rounds $t \in \dsb{0,T-1}$:
\begin{align}
    \left\| \widehat{\vgamma}_t(a) - \vgamma(a) \right\|_{\Vs_t(a)} \leq \beta_t(a). \label{eq:ball}
\end{align}
We observe that $\beta_t(a)$ (see Equation~\ref{eq:beta}) and \algnameshort do not require the knowledge of the maximum sum $\Gamma$ of the parameters $\gamma_i(a)$ over the actions (see Assumption~\ref{ass:stability}). This is a desirable feature of our algorithm as  $\Gamma$ is often unknown in practice and difficult to upper bound or estimate. Nevertheless, $\Gamma$ appears in the regret analysis in Section~\ref{sec:regretAnalysis}. Differently, the value of $m$, needed to compute the optimistic coefficient $\beta_t(a)$, can be easily replaced with an upper bound $\overline{m} > m$ when unknown.\footnote{An empirical analysis of the effect of the misspecification of such a parameter is provided in Section~\ref{subsec:m_misspec}.}

\subsection{Regret Decomposition}\label{sec:regretAnalysis}
In this section, we present a novel \emph{decomposition} of the regret that will be employed in the final bound of Section~\ref{sec:boundRegret}. The contents of this section are of independent interest and applicable to any learner's policy $\bm{\pi}$, beyond \algnameshort. From a technical perspective, the analysis is composed of two steps: ($i$) we decompose the instantaneous (policy) regret $r_t$ in terms of the instantaneous \emph{external regret} $\rho_t$ (Lemma~\ref{lem:decomposition}); ($ii$) we bound the cumulative expected  (policy) regret $R(\bm{\pi},T) =\mathbb{E} [ \sum_{t=1}^T r_t ]$  in terms of the expected cumulative external regret $\varrho(\bm{\pi},T) = \mathbb{E} [ \sum_{t=1}^T \rho_t ]$ (Lemma~\ref{lem:internal_regret_bound}).

We start with step ($i$), by recalling that the definition of cumulative expected (policy) regret  $R(\bm{\pi},T)$ in Equation~\eqref{eq:regretDef}  compares the sequence of rewards $(x_t^*)_{t \in \dsb{T}}$ when executing the optimal policy $\bm{\pi}^*$ with the sequence of rewards $(x_t)_{t \in \dsb{T}}$ when executing the learner's policy $\bm{\pi}$. However, in our \settingnameshort setting, the observed reward $x_t$ depends on the past history $H_{t-1}$. Thus, the instantaneous (policy) regret $r_t \coloneqq x_t^*-x_t$ can be decomposed in two terms: ($a$) the dissimilarity between the past history $H_{t-1}^*$ when executing the optimal policy and the learner's observed history $H_{t-1}$; ($b$) the instantaneous \emph{external regret}~\citep{dekel2012online} $\rho_t \coloneqq \inner{\vgamma(a_t^*) - \vgamma({a}_t)}{\zs_{t-1}}$ representing the loss of executing the learner action $a_t $ instead of the optimal one $a_t^* = \pi^*_t(H_{t-1}^*)$ \emph{assuming} that such actions are applied to the observations vector $\zs_{t-1}$ generated by the execution of the learner's policy. The following result formalizes the instantaneous regret decomposition.

\begin{restatable}[Policy Regret Decomposition]{lemma}{Decomposition}\label{lem:decomposition}
Let $(x_t^*)_{t \in \dsb{T}}$ be the sequence of rewards by executing the optimal policy $\bm{\pi}^*$ and let $(x_t)_{t \in \dsb{T}}$ be the sequence of rewards by executing the learner's policy $\bm{\pi}$. Then, for every $t \in \dsb{T}$ it holds that:
\begin{align}
    r_t & = x^*_t - x_t \nonumber \\
    & = \sum_{i=1}^k \gamma_i(a_t^*) (x^*_{t-i} - x_{t-i}) + \inner{\vgamma(a_t^*) - \vgamma({a}_t)}{\zs_{t-1}} \nonumber \\
    &= \sum_{i=1}^k \gamma_i(a_t^*) r_{t-i} + \rho_t,\label{eq:regret_decomposition}
\end{align}
where $r_t \coloneqq x_t^*-x_t$ is the instantaneous policy regret, $\rho_t \coloneqq \inner{\vgamma(a_t^*) - \vgamma({a}_t)}{\zs_{t-1}}$ is the instantaneous external regret, $a_t^* = \pi_t^*(H_{t-1}^*)$, and $r_{t-i} = 0$ if $i \ge t$.
\end{restatable}

The decomposition in Equation~\eqref{eq:regret_decomposition} comprises two terms. The second one $\rho_t$ is the instantaneous external regret discussed above. The first one defines a recurrence relation of order $k$ on the instantaneous policy regret $r_t$. We now move to step ($ii$) with the following result that shows that the contribution of the recurrence can be reduced to a term depending on $\Gamma$ and $k$ that multiplies the cumulative external regret.

\begin{restatable}[External-to-Policy Regret Bound]{lemma}{InternalRegretBound}\label{lem:internal_regret_bound}
    Let $\bm{\pi}$ be the learner's policy and $T \in \mathbb{N}$ be the horizon. Under Assumptions~\ref{ass:monotonicity} and~\ref{ass:stability}, it holds that:
      \begin{align}
         R(\boldsymbol\pi,T) & = \mathbb{E} \left[\sum_{t=1}^T\left[\sum_{i=1}^k \gamma_i(a_t^*) r_{t-i} + \rho_t \right]\right] \nonumber \\
         & \le \left(\frac{\Gamma k}{1-\Gamma} + 1\right)  \varrho(\bm{\pi},T), \label{eq:internal_regret_bound}
    \end{align}
    where $\varrho(\bm{\pi},T) \coloneqq \mathbb{E} \left[ \sum_{t=1}^T \rho_t \right]$ is the cumulative expected external regret.
\end{restatable}  
Lemma~\ref{lem:internal_regret_bound} provide us a bound on the cumulative expected (policy) regret $R(\boldsymbol\pi,T)$ achieved by \algnameshort (or any algorithm playing in an ARB) by bounding the cumulative expected external regret $\varrho(\bm{\pi},T)$. The order of the regret bound \wrt $T$ is governed by the external regret, while the effect of a \emph{weaker} history (\ie the sub-optimal actions of the past) emerges as an instance-specific constant. Such a constant is $1$ whenever $k=0$ or $\Gamma = 0$, \ie when the ARB reduces to a standard MAB. In all other cases, the bigger the value of $k$ or $\Gamma$, the more visible the AR effects are, and, consequently, the more the sub-optimal choices of the past get amplified. Finally, we point out that the multiplicative factor $\frac{\Gamma k}{1-\Gamma}+1$ to pass from external to policy regret is tight since there exists a sequence of external regrets in which the inequality of Lemma~\ref{lem:internal_regret_bound} holds with equality (see Appendix~\ref{apx:omitted_proofs}).

\subsection{Regret Bound}\label{sec:boundRegret}
In the following, we present a bound on the expected policy regret bound for \algnameshort.
\begin{restatable}[]{thr}{RegretBound}\label{thr:RegretBound}
Let $\delta = (2T)^{-1}$. Under Assumptions~\ref{ass:monotonicity}, \ref{ass:stability}, and~\ref{ass:boundedness}, \emph{\algnameshort} suffers a cumulative expected (policy) regret bounded by (highlighting the dependence on $m$, $\sigma$, $k$, $\Gamma$, $n$, and $T$):
\begin{align*}
 \mathbb{E} [ R(\text{\emph{\algnameshort}},T) ] \le \widetilde{\mathcal{O}}\bigg(\frac{(m + \sigma)(k+1)^{3/2}\sqrt{nT}}{(1-\Gamma)^2}\bigg).
\end{align*}
\end{restatable}

Some observations are in order.
First, when we set $k=0$ and $\Gamma=0$, \ie we reduce the ARB to a standard MAB, we obtain a regret rate of $\widetilde{\mathcal{O}}((m+\sigma)\sqrt{nT})$, which is tight for standard MABs. The quantity $\frac{m+\sigma}{1-\Gamma}$ is the maximum value that rewards can achieve, as proven in Lemma \ref{lem:bounded_sequences}. As intuition suggests, the ARB learning problem becomes more challenging as the AR order $k$ increases and when the bound on the sum of the parameters $\Gamma$ approaches one. This is witnessed in Theorem~\ref{thr:RegretBound} with  the dependence of the regret on $(k+1)^{3/2}$ and $(1-\Gamma)^{-1}$. The interplay between $k$ and $(1-\Gamma)^{-1}$ shows that even if two instances have the same sum of parameters (\ie $\Gamma$), the one with fewer coefficients (\ie $k$) is more easily learnable. This is explained by the fact that our algorithm learns the individual parameters by means of a regression procedure learning to a $\sqrt{k+1}$ in the regret.
Finally, suppose we run \algnameshort with a larger AR order $\overline{k} > k$. In such a case, the dependence on $(k+1)^{3/2}$ becomes $(k+1)(\overline{k}+1)^{1/2}$, since the factor due to passing from external to policy regret (Lemma~\ref{lem:internal_regret_bound}) will always contain the true $k$, while $\overline{k}$ appears because of the estimation process. Similarly, if we execute \algnameshort with a value $\overline{m} > m$, the regret bound still holds by replacing $m$ with $\overline{m}$.

\begin{remark}[Comparison with MDPs]
If we consider our problem as an MDP, we are in an undiscounted infinite horizon scenario. This scenario is more challenging w.r.t.~the episodic one. Regret minimization in infinite horizon MDPs has been studied in very few cases: Tabular, LQR, and H\"older continuous MDPs~\citep{ortner2012online}. The \settingnameshort setting is not tabular (as it has continuous space) nor an LQR (as it has discrete actions). Our setting can be viewed as an H\"older continuous MDP by making a one-hot encoding of the $n$ actions, but the regret bounds for this family of processes are, in the best-case scenario, in the order of $\widetilde{\mathcal{O}} \left( T^{2/3} \right)$, much worse than our bound of order $\widetilde{\mathcal{O}} \left( \sqrt{T} \right) $.
\end{remark}
\section{NUMERICAL VALIDATION}
\label{sec:experiments}

In this section, we first provide (Section~\ref{subsec:exp_1}) a numerical validation of \algnameshort compared with other bandit baselines in synthetically-generated domains. Then, we discuss (Section~\ref{subsec:exp_constant}) the importance of exploiting the noise in this setting, and, subsequently, we analyze the sensitivity of \algnameshort to the misspecification of the two most important parameters, \ie $m$ (Section~\ref{subsec:m_misspec}) and $k$ (Section~\ref{subsec:k_misspec}).
Additional experimental results are provided in Appendix~\ref{apx:additionalexperiments}. 
The code to reproduce the experiments can be found at \url{https://github.com/gianmarcogenalti/autoregressive-bandits}.

\textbf{Running Time}~~~The algorithms are implemented in Python~$3.11$, and run over an Intel Core $i7-8750H$ @ $2.20$~GHz with $16$~GB DDR4 RAM. All the presented experiments took $\approx 10$ minutes for a complete run.

\subsection{\algnameshort vs Bandit Baselines}
\label{subsec:exp_1}

\textbf{Setting}~~~We evaluate \algnameshort in three scenarios that differ in the properties of the autoregressive processes that govern the rewards. The competing algorithms are evaluated in terms of cumulative regret w.r.t.~the setting-specific clairvoyant. The three settings have their AR($k$) process order $k\in\{2,4\}$, number of actions $n\in\{2,7\}$, and scale ${m}\in\{1,20,920\}$. The values of $\vgamma(a)$ have been sampled from uniform probability distributions for each action $a\in\mathcal{A}$ and for each setting. The environments are noisy with a standard deviation $\sigma\in\{0.75,1.5,10\}$. We chose to set the hyper-parameters of \algnameshort as follows: $\lambda=1$, while $\overline{m} \in \{10,100,1000\}$, that is equivalent to chose $\overline{m}$ of the same magnitude of the true value $m$, in a pessimistic fashion. Table~\ref{tab:experimental_settings_1} 
summarizes the details of the three environments.

\textbf{Baselines}~~~\algnameshort will compete with several bandit baselines. First, it is compared with \ucbone~\citep{auer2002finite}, a widely adopted solution for stochastic MABs. Second, we consider \expthree, designed for adversarial MABs~\citep{auer1995gambling,auer2002nonstochastic} and its extension to finite-memory adaptive adversaries \batchexpthree~\citep{dekel2012online}.
Lastly, we compare \algnameshort with \artwo~\citep{chen2021dynamic}, an algorithm for managing AR($1$) processes. The hyper-parameters chosen for the baselines are the ones proposed in the original papers.

\begin{table}
\centering
\resizebox{0.6\linewidth}{!}{\setlength{\tabcolsep}{5pt}
    \begin{tabular}{c|cccc}
    \toprule
    \textbf{} & \multicolumn{3}{c}{\textbf{Parameters}} \\
    \textbf{Setting} & $k$ & $n$ & $m$ & $\sigma$ \\
    \midrule
    \texttt{A} & $2$ & $2$ & $1$ & $0.75$ \\
    \texttt{B} & $4$ & $7$ & $20$ & $1.5$ \\
    \texttt{C} & $4$ & $7$ & $920$ & $10$ \\ 
    \bottomrule
    \multicolumn{4}{c}{ $ $ } \\
    \end{tabular}
}
\caption{Settings description.}
\label{tab:experimental_settings_1}
\end{table}

\begin{figure*}[t!]
    \centering
    \subfloat[Setting \texttt{A}.]{\resizebox{0.3\linewidth}{!}{\includegraphics[]{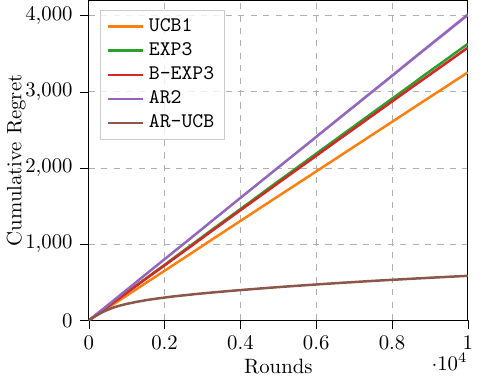}} \label{fig:testcase_6}}
    \hfill
    \subfloat[Setting \texttt{B}.]{\resizebox{0.3\linewidth}{!}{\includegraphics[]{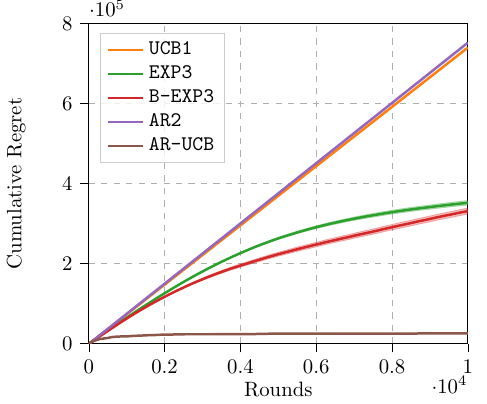}} \label{fig:testcase_8}}
    \hfill
    \subfloat[Setting \texttt{C}.]{\resizebox{0.3\linewidth}{!}{\includegraphics[]{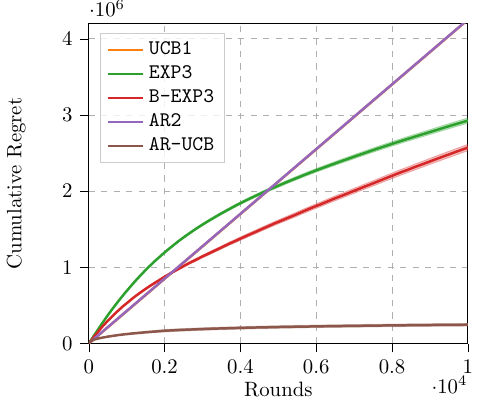}} \label{fig:testcase_11}}
    \caption{Settings and cumulative regret of \algnameshort and multiple baselines (100 runs, mean $\pm$ std).}
    \label{fig:cum_regrets}
\end{figure*}

\textbf{Results}~~~Figure~\ref{fig:cum_regrets} shows the average cumulative regrets for \algnameshort and the other bandit baselines. We observe that \algnameshort suffers the smallest cumulative regret in these scenarios, always displaying a sublinear behavior. Both \expthree and \batchexpthree in two scenarios out of three (\texttt{B} and \texttt{C}) achieve sublinear regret. On the other hand, both \ucbone and \artwo are not able to achieve sublinear regret in the presented scenarios. This is not surprising since we require them to learn more complex processes than those they are designed for (\ie models with $k=0$ and $k=1$ for \ucbone and \artwo, respectively).

\subsection{On the Effect of Stochasticity}
\label{subsec:exp_constant}

The optimal policy (Theorem~\ref{thr:optimal_policy}) for the \settingnameshort setting exploits the contribution of the noise to increase the collected reward. In this section, we provide experimental evidence of this phenomenon. We first introduce a notion of \emph{optimal policy without noise}.  
Then, we conduct an experiment to highlight the variations between the two policies in environments presenting different noise magnitudes. 

\textbf{Optimal Policy without Noise}~~~The optimal policy, when no noise is involved, is \emph{constant} and corresponds, for sufficiently large $T$, to playing the action $a^+ \in \mathcal{A}$ that brings the system to the most profitable steady state.\footnote{The request for large $T$ is to make transient effects neglectable.} Such an action $a^+$ is the one maximizing the \textit{steady-state reward}, namely:
\begin{equation}
a^+ \in \argmax_{a \in \mathcal{A}} \frac{\gamma_0(a)}{1 - \sum_{i=1}^k\gamma_i(a)}.
\label{eq:optimal_constant_policy}
\end{equation}
It is worth noting the role of Assumption~\ref{ass:stability} which guarantees the existence of the inverse $(1-\sum_{i=1}^k\gamma_i(a))^{-1} \ge (1-\Gamma)^{-1}$ for each action $a \in \mathcal{A}$. The proof can be found in Appendix~\ref{apx:optimal_policy_no_noise}.

\textbf{Setting}~~~To demonstrate the importance of the noise in this setting, we consider the two clairvoyant policies defined above. We compare the optimal \texttt{Stochastic} policy (Equation~\ref{eq:opt_policy}) and the optimal policy for the \texttt{Deterministic} setting (Equation~\ref{eq:optimal_constant_policy}).
The setting selected is challenging and made of $k=2$ actions, $a_1$ and $a_2$, that are very close in terms of expected steady-state reward:
\begin{align*}
    &\vgamma(a_1) = (1, \; \rho, \; 0)^T \,\
    &\vgamma(a_2) = (1, \; 0, \;\rho-\epsilon)^T,
\end{align*}
where $\rho=0.5$, $\epsilon=0.02$ and the noise is Gaussian with $\sigma \in \{ 0, \, 0.1, \,  0.5, \, 1.0, \, 2.0 \}$.

\begin{table}[t]
\centering
\resizebox{0.75\linewidth}{!}{
\begin{tabular}{c|cc}
\toprule
 $\sigma$ & \texttt{Stochastic} & \texttt{Deterministic} \\
\midrule
 0 & \textbf{19994 (0)} & \textbf{19994 (0)} \\
 0.1 & \textbf{20167 (0.20)} & 19998 (2.04)  \\
 0.5 & \textbf{22049 (1.02)}  & 20012 (1.02)  \\
 1.0 & \textbf{24504 (2.04)}  & 20030 (2.04) \\
 2.0 & \textbf{29428 (4.09)} & 20067 (4.08) \\
\bottomrule
\multicolumn{3}{c}{ $ $ } \\
\end{tabular}}
\caption{Cumulative reward of the \texttt{Stochastic} and \texttt{Deterministic} clairvoyants ($100$ runs, mean (std)).}
\label{tab:constant_vs_optimal}
\end{table}

\textbf{Results}~~~Table~\ref{tab:constant_vs_optimal} shows the performance of the two policies in terms of cumulative reward. First, with no noise (\ie $\sigma = 0$), the performances of the two policies are equivalent. However, when we consider a stochastic setting (\ie $\sigma > 0$), the \texttt{Stochastic} policy can exploit the beneficial effect of the noise in order to increase the average reward. Indeed, the optimal \texttt{Deterministic} policy retrieves almost the same reward for all the tested values of $\sigma$, while \texttt{Stochastic} policy increases its average reward as much as the system is noisy (since it can exploit it).

\begin{figure}[t]
    \centering
    \includegraphics[width=0.66\linewidth]{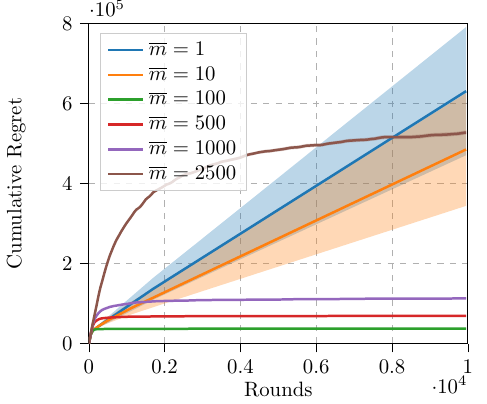}
    \caption{Effect of the choice of parameter $\overline{m}$ on the \algnameshort cumulative regret ($100$ runs, mean $\pm$ std).}
    \label{fig:m_bound_analysis}
\end{figure}

\subsection{On the Knowledge of Parameter $m$}
\label{subsec:m_misspec}
A fundamental parameter of \algnameshort is the value $m = \max_{a\in\mathcal{A}}\gamma_0(a)$. In this part, we empirically show that any choice in the same order of magnitude as the actual value will let the algorithm achieve a sublinear regret, while severe underestimation prevents the algorithm from achieving a sublinear cumulative regret.

\textbf{Setting}~~~We run multiple simulations varying the value of parameter $\overline{m}$. We chose $n=7$, $k=4$ and $\gamma_0(a)=500$ for every action $a\in\mathcal{A}$ (\ie $m=500$). The autoregressive parameters $\gamma_i(a)$ have been sampled from a uniform probability distribution with support in $[0,1/4-\epsilon]$, where $\epsilon>0$ is an arbitrarily small value. For this experiment, we test values $\overline{m} \in \{1, 10, 100, 500, 1000, 2500\}$. 

\textbf{Results}~~~In Figure~\ref{fig:m_bound_analysis}, we report the cumulative regret of \algnameshort under different choices of $\overline{m}$. First, it is worth noting how choosing values of $\overline{m} \ge m$ always results in a sublinear cumulative regret, with a progressive increase as $\overline{m}$ gets larger. This is highlighted when comparing the scenario where $\overline{m}=2500$ to the one where $\overline{m} \in \{500, 1000\}$. When $\overline{m}$ is underestimated, we empirically observe two facts. When $\overline{m}$ is in the same order of magnitude as the true value $m$ (\eg $\overline{m}=100$), we empirically observe a smaller sublinear cumulative regret. Instead, a severe underestimation of the parameter leads to a linear cumulative regret, as clearly visible for $\overline{m} \in \{1,10\}$, although, in these settings, the cumulative regret is lower w.r.t.~the other settings in the very first stages of the simulations (due to a more limited exploration).

\begin{figure}[t]
    \centering
    \includegraphics[width=0.66\linewidth]{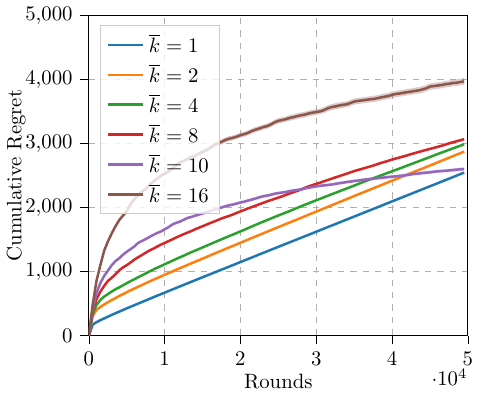}
    \caption{Effect of the choice of parameter $\overline{k}$ on the \algnameshort cumulative regret ($100$ runs, mean $\pm$ std).}
    \label{fig:k_analysis}
\end{figure}

\subsection{On the Knowledge of the Autoregressive Order $k$}
\label{subsec:k_misspec}
As discussed in Section~\ref{sec:regret}, \algnameshort can also run under a misspecified parameter $\overline{k} \neq k$. We now empirically study the effect of misspecifying such a value. 

\textbf{Setting}~~~We consider a configuration with $n=7$, $k=10$, $\gamma_0(a)=1$ and $\gamma_i(a)$ for $i \ge 1$ sampled from a uniform distribution having support in $[0, 10^{-2}\cdot 2i)$ for every action $a\in\mathcal{A}$.
\algnameshort is run varying the parameter $\overline{k} \in \{1,2,4,8,10,16\}$.

\textbf{Results}~~~Figure~\ref{fig:k_analysis} reports the average cumulative regret for the considered values of $\overline{k}$. On the one hand, an underestimation of parameter $k$ (\ie $\overline{k}\in\{1,2,4\}$) results in an asymptotically linear cumulative regret. This effect is justified since \algnameshort is not able to learn the actual AR dynamics due to underfitting, \ie the considered models are too simple. On the other hand, \algnameshort achieves sublinear cumulative regret when $\overline{k} \ge k$ (\ie $\overline{k}\in\{10,16\}$). In particular, when $\overline{k}>k$, the linear models use more parameters than required, resulting in slower learning. However, as the samples increase, the algorithm learns that the exceeding coefficients are not significant. 
A particular case is when $\overline{k}$ is close to $k$ but strictly lower (\ie $\overline{k} = 8$). Here, the cumulative regret degenerates to linear, but if the coefficients $\gamma_j(a)$ for $j\in\dsb{\overline{k}+1,k}$ are not very large, the performance of \algnameshort with misspecified $\overline{k}$ results, in practice, close to the one obtained with the true $k$.
\section{RELATED WORKS}
\label{sec:relatedworks}
In this section, we discuss and compare the works that share similarities with the Autoregressive Bandits. We analyze both solutions related to multi-armed bandits and online learning in non-linear systems.

\textbf{Multi-Armed Bandits}~~~In the more classical Multi-Armed Bandit (MAB) setting, the learning problem does not involve temporal dependencies between rewards. The MAB setting has been studied under the assumptions of both \emph{stochastic} and \emph{adversarial} noise models. In the former case, \ucbone~\citep{lai1985asymptotically,auer2002finite} represents the parent algorithm. Instead, when adversarial noise is involved \expthree~\citep{auer1995gambling,auer2002nonstochastic} is usually employed. This algorithm has been extended by \texttt{R}\expthree~\citep{besbes2014stochastic} to handle with the \emph{non-stationary} setting. 
Differently from both the adversarial and non-stochastic setting, we assume that the rewards are not preselected by an adversary or nature but, instead, they change as an effect of the actions played. Indeed, the underlying autoregressive process (affected by a stochastic noise) is such that the current action impacts the future rewards. Therefore, importing the adversarial MAB terminology, the ARBs can be reduced to an adversary setting with an \emph{adaptive} (or non-oblivious) adversary~\citep{dekel2012online}. In particular, the $\mathcal{O}(\sqrt{nT})$ regret guarantees of \expthree are not achievable in the ARB setting as \expthree competes against the best constant policy while the optimal policy for ARBs is not constant (see Theorem~\ref{thr:optimal_policy} and Section~\ref{sec:experiments}). 
Moreover, our setting presents similarities with MABs with \emph{delayed} feedback~\citep[\eg][]{pike2018bandits}. However, in \settingnameshort the effect of the actions is propagated (not exactly delayed). Markov~\citep{ortner2012regret} and restless~\citep{tekin2012online} bandits, instead, consider underlying processes that influence the rewards. However, these processes are not supposed to be controlled by the action history. Other works~\citep[\eg][]{mussi2023dynamical} consider complex action-dependent feedback vanishing over time. In \citet{chen2021dynamic}, the authors study the control problem in a setting that considers temporal structure modeled as an AR(1) process.

\textbf{Online Learning in Non-Linear Systems}~~~The \settingnameshort setting is a specific case of a non-linear dynamical system. Although the literature related to this setting is wide, no work faces all problems that the ARB setting presents, including learning to control with regret guarantees. \citet{mania2022active} focus on learning the parameters of a particular class of non-linear systems. However, the approach is limited to estimation and no control algorithm is proposed. Similarly, \citet{umlauft2017learning} deal with learning the system parameters with stability guarantees without the chance to control it. Several recent works~\citep[\eg][]{kakade2020information,lale2021model} focus on the learning and control of non-linear systems with regret guarantees. However, these works make use of an oracle to solve a complex optimization problem to perform optimistic planning (\ie optimal policy given an optimistic estimate of the system). This problem in a non-linear setting, however, is proven to be NP-hard~\citep{sahni1974computationally,dani2008stochastic}. Furthermore, the class of non-linear systems considered in these works does not include the ARB setting. Other works~\citep[\eg][]{albalawi2021regret} overcome the request for the oracle by searching in the restricted space of constant policies, leading to the best equilibrium. However, this solution can be suboptimal in several cases, including ARBs (see Section~\ref{subsec:exp_constant}).
\section{CONCLUSIONS}
\label{sec:conclusions}

In this work, we faced the online sequential decision-making problem where an autoregressive temporal structure between the observed rewards is present.
First, we formally introduced the \settingnameshort setting and defined the notion of optimal policy, demonstrating that a myopic policy is optimal also to optimize the total reward, regardless of the target time horizon, and that the optimal policy is not constant over time and depends on the realizations of the reward. Then, we proposed an optimistic bandit algorithm, \algnameshort, to learn online the parameters of the underlying process for each action. We demonstrated that the presented algorithm enjoys sublinear regret, depending on the AR order $k$ and on an index of the speed at which the system reaches a stable condition. Finally, we provided an experimental campaign to validate the proposed solution, and we analyzed the behavior of \algnameshort when key parameters are misspecified. 
Future directions should focus on fully understanding the complexity of learning in the \settingnameshort setting, deriving tight lower bounds, and matching algorithms.

\section*{ACKNOWLEDGMENTS}
This paper is supported by PNRR-PE-AI FAIR project funded by the NextGeneration EU program.

\bibliography{biblio}
\bibliographystyle{unsrtnat}


\onecolumn
\appendix
\section{OMITTED PROOFS}
\label{apx:omitted_proofs}

\OptimalPolicy*
\begin{proof}
We first prove an intermediate result auxiliary to get to the final statement. Let us denote with $J^*_T(\zs)$ the expected cumulative reward when the initial observations vector is $\zs=(1,x_0,x_{-1},\dots,x_{-k+1})$. Let us denote with $\succeq$ the element-wise inequality. We show that for every $T \in \mathbb{N}$, if $\zs \succeq \overline{\zs}$, then $J^*_T(\zs) \ge J^*_T(\overline{\zs})$. 

We proceed by induction. 

For $T=1$, we have $J^*_1(\zs) = \max_{a \in \As} \inner{\vgamma(a)}{\zs} = \inner{\vgamma(a_1^*)}{\zs}$, where $a_1^* \in \argmax_{a \in \As}\inner{\vgamma(a)}{\zs}$ and $J^*_1(\overline{\zs}) = \max_{a \in \As} \inner{\vgamma(a)}{\overline{\zs}} = \inner{\vgamma(\overline{a}_1^*)}{\overline{\zs}}$, where $\overline{a}_1^* \in \argmax_{a \in \As}\inner{\vgamma(a)}{\overline{\zs}}$. Thus, we have:
\begin{align*}
	J^*_1(\zs) = \inner{\vgamma(a_1^*)}{\zs} \ge \inner{\vgamma(\overline{a}_1^*)}{\zs} \stackrel{\text{(a)}}{\ge}  \inner{\vgamma(\overline{a}_1^*)}{\overline{\zs}} = J^*_1(\overline{\zs}),
\end{align*}
where inequality (a) follows from Assumption~\ref{ass:monotonicity}.

Suppose the statement holds for $T-1$, we prove it for $T > 1$. 
To this end, we consider the \emph{transition operator} $P : \mathcal{Z} \times \As \times \mathbb{R} \rightarrow \mathcal{Z}$, defined for every observations vector $\zs_t = (1,x_{t-1},x_{t-2},\dots,x_{t-k}) \in \mathcal{Z}$, action $a \in \As$, and noise $\xi \in \mathbb{R}$ as follows:
\begin{align*}
	P(\zs_t,a,\xi) =  P \left( \begin{pmatrix} 1 \\ x_{t-1} \\ x_{t-2} \\ \vdots \\ x_{t-k} \end{pmatrix}, a, \xi \right) = \begin{pmatrix} 1 \\ x_t \\ x_{t-1} \\ \vdots \\ x_{t-k+1} \end{pmatrix}	 = \zs_{t+1}, \qquad \text{ where } \qquad x_t = \inner{\vgamma(a)}{\zs_t} + \xi.
\end{align*}
Thus, we can look at the stochastic process as a Markov decision process~\citep{puterman2014markov} with $\zs_t$ as state representation.
We immediately observe that if $\zs \succeq \overline{\zs}$, we have that $P(\zs,a,\xi) \succeq P(\overline{\zs},a,\xi)$, for every action $a \in \As$ and noise $\xi \in \mathbb{R}$. By applying the Bellman equation, we obtain:
\begin{align*}
	&J^*_T(\zs) = \max_{a \in \As} \left\{ \inner{\vgamma(a)}{\zs} + \mathbb{E}_{\xi_T}\left[J^*_{T-1}(P(\zs, a, \xi_T)) \right] \right\} = \inner{\vgamma(a^*_T)}{\zs} + \mathbb{E}_{\xi_T}\left[J^*_{T-1}(P(\zs, a^*_T, \xi_T))\right],\\
	&J^*_T(\overline{\zs}) = \max_{a \in \As} \left\{ \inner{\vgamma(a)}{\overline{\zs}} + \mathbb{E}_{\xi_T}\left[J^*_{T-1}(P(\overline{\zs}, a, \xi_T)) \right] \right\} = \inner{\vgamma(\overline{a}^*_T)}{\overline{\zs}} + \mathbb{E}_{\xi_T}\left[J^*_{T-1}(P(\overline{\zs}, \overline{a}^*_T, \xi_T))\right],
\end{align*}
where the actions are defined as $a^*_T \in \argmax_{a \in \As} \left\{ \inner{\vgamma(a)}{\zs} + \mathbb{E}_{\xi_T}\left[J^*_{T-1}(P(\zs, a, \xi_T)) \right]  \right\}$ and $\overline{a}^*_T \in \argmax_{a \in \As} \left\{ \inner{\vgamma(a)}{\overline{\zs}} + \mathbb{E}_{\xi_T}\left[J^*_{T-1}(P(\overline{\zs}, a, \xi_T)) \right] \right\}$. Thus, we have:
\begin{align*}
	J^*_T(\zs) & = \inner{\vgamma(a^*_T)}{\zs} + \mathbb{E}_{\xi_T}\left[J^*_{T-1}(P(\zs, a^*_T, \xi_T))\right] \\
	& \ge \inner{\vgamma(\overline{a}^*_T)}{\zs} + \mathbb{E}_{\xi_T}\left[J^*_{T-1}(P(\zs, \overline{a}^*_T, \xi_T))\right] \\
	& \stackrel{(b)}{\ge}  \inner{\vgamma(\overline{a}^*_T)}{\overline{\zs}} + \mathbb{E}_{\xi_T}\left[J^*_{T-1}(P(\overline{\zs}, \overline{a}^*_T, \xi_T))\right] = J^*_T(\overline{\zs}),
\end{align*}
where (b) follows from Assumption~\ref{ass:monotonicity} when bounding $\inner{\vgamma(\overline{a}^*_T)}{{\zs}} \ge \inner{\vgamma(\overline{a}^*_T)}{\overline{\zs}}$ and by observing that $P(\zs, \overline{a}^*_T, \xi_1) \succeq P(\overline{\zs}, \overline{a}^*_T, \xi_T)$ and, then, exploiting the inductive hypothesis.

We conclude that the optimal policy is the myopic one by observing that both $\inner{\bm{\gamma}(a)}{\bm{z}} $ and $J_{T-1}^*(P(\zs,a,\xi))$ are simultaneously maximized by $\argmax_{a \in \mathcal{A}} \inner{\bm{\gamma}(a)}{\bm{z}} $.
\end{proof}

\Concentration*

\begin{proof}
    We consider an action at a time; then, the final result is obtained with a union bound over $\mathcal{A} = \dsb{n}$. Let $a \in \mathcal{A}$. We first observe that the estimates of action $a$ change only when $a$ is pulled. Let $l \in \mathbb{N}$ be an index and let $t_l(a) \in \mathbb{N}$ be the round in which action $a$ is pulled for the $l$-th time, \ie $ \{t_l(a) \,:\, l \in \mathbb{N}\} = \mathcal{O}_{\infty}(a)$. Thus, we have: 
    \begin{align*}
        \vgamma_{t_l}(a) & = \Vs_{t_l(a)}^{-1}(a) \bs_{t_l(a)}^{-1}(a) \\
        & = \left(\lambda \Is_{k+1} + \sum_{j=1}^l \zs_{t_j(a)-1}\zs_{t_j(a)-1}^T \right)^{-1} \sum_{j=1}^l \zs_{t_j(a)-1} x_{t_j} \\
        & = \left(\lambda \Is_{k+1} + \sum_{j=1}^l \zs_{t_j(a)-1}\zs_{t_j(a)-1}^T \right)^{-1} \sum_{j=1}^l \zs_{t_j(a)-1} \left( \inner{\vgamma(a)}{\zs_{t_j(a)-1}} + \xi_{t_j(a)} \right) \\
        & \stackrel{\text{(a)}}{=} \vgamma(a) - \lambda \left(\lambda \Is_{k+1} + \sum_{j=1}^l \zs_{t_j(a)-1}\zs_{t_j(a)-1}^T \right)^{-1} \vgamma(a) + \\
        & \qquad\qquad\qquad + \left(\lambda \Is_{k+1} + \sum_{j=1}^l \zs_{t_j(a)-1}\zs_{t_j(a)-1}^T \right)^{-1} \sum_{j=1}^l \zs_{t_j(a)-1} \xi_{t_j(a)} \\
        & = \vgamma(a) - \lambda \Vs_{t_l(a)}^{-1}(a) \vgamma(a) + \Vs_{t_l(a)}^{-1}(a) \underbrace{\sum_{j=1}^l \zs_{t_j(a)-1} \xi_{t_j(a)}}_{\textbf{s}_{t_l(a)}},
    \end{align*}
    where the passage (a) derives from the observation that
    $\sum_{j=1}^l \zs_{t_j-1} ( \inner{\vgamma(a)}{\zs_{t_j-1}})=\sum_{j=1}^l \zs_{t_j-1}\zs_{t_j-1}^T\vgamma(a)$.
    Thus, we have:
    \begin{align*}
        \left\| \vgamma_{t_l(a)}(a) - \vgamma(a) \right\|_{\Vs_{t_l(a)}(a)} \le \sqrt{\lambda} \| \vgamma(a) \|_2  + \|\textbf{s}_{t_l(a)}\|_{\Vs_{t_l(a)}^{-1}(a)}.
    \end{align*}
     Let us denote with $\mathcal{F}_{t_l(a)} = \sigma(\zs_0, a_1, \zs_1, a_2, \dots, \zs_{t_l(a)-1}, a_{t_l(a)})$ be the filtration generated by all events realized at round $t_l(a)$. Let us now consider the stochastic processes $(\xi_{t_l(a)})_{l \in \mathbb{N}}$ and  $(\zs_{t_{l}(a)-1})_{l \in \mathbb{N}}$. We observe that $\xi_{t_l(a)}$ is $\mathcal{F}_{t_{l}(a)}$-measurable and conditionally $\sigma^2$-subgaussian and that $\zs_{t_{l}(a)-1}$ is $\mathcal{F}_{t_l(a)-1}$-measurable. By applying Theorem~1 of~\citet{abbasi2011improved}, we have that simultaneously for all $l \in \mathbb{N}$, w.p. $1-\delta$:
    \begin{align*}
       \|\textbf{s}_{t_l(a)}\|_{\Vs_{t_l(a)}^{-1}(a)} \le \sigma \sqrt{2\log \frac{1}{\delta} + \log \frac{\det \Vs_{t_l(a)}(a)}{\lambda^{k+1}}}.
    \end{align*}
    Clearly, this hold for the rounds $t \in \mathbb{N}$ in which the action $a$ is not pulled, since the corresponding estimates do not change.
\end{proof}

\Decomposition*

\begin{proof}
Let $t \in \dsb{T}$ and let us denote with $\zs_{t-1}^*=(1,x^*_{t-1}, \dots, x_{t-k}^*)^T$ the observations vector associated with the execution of the optimal policy and with $\zs_{t-1}=(1,x_{t-1}, \dots, x_{t-k})^T$ the observations vector associated with the execution of the learner's policy. We have:
\begin{align*}
    r_t &=   x^*_t - x_t\\
    &=\inner{\vgamma(a_t^*)}{\zs_{t-1}^*}-\inner{\vgamma( a_t)}{\zs_{t-1}}\\
    & = \inner{\vgamma(a_t^*)}{\zs_{t-1}^*}-\inner{\vgamma(a_t^*)}{\zs_{t-1}}+\inner{\vgamma(a_t^*)}{\zs_{t-1}}-\inner{\vgamma( a_t)}{\zs_{t-1}}\\
    & = \inner{\vgamma(a_t^*)}{\zs_{t-1}^*-\zs_{t-1}}+\inner{\vgamma(a_t^*)-\vgamma( a_t)}{\zs_{t-1}}\\
    & =  \sum_{i=1}^k \gamma_i(a_t^*) \underbrace{(x^*_{t-i} - x_{t-i})}_{r_{t-i}} + \underbrace{\inner{\vgamma(a_t^*) - \vgamma({a}_t)}{\zs_{t-1}}}_{\text{$\rho_t$}},
\end{align*}
where in expanding the inner product we made the summation start from $i=1$ as the two vectors $\zs_{t-1}^*$ and $\zs_{t-1}$ have the same first component  equal to $1$.
\end{proof}

\InternalRegretBound*

\begin{proof}
    We start from the decomposition of Lemma~\ref{lem:decomposition}. To prove the result we employ the so-called \quotes{superposition principle}, which allows us to decompose the linear recurrence as follows:
    $$r_t = \sum_{i=1}^k \vgamma_i(a_t^*) r_{t-i} + \rho_t = \sum_{\tau=0}^{+\infty} \rho_\tau \widetilde r_{t,\tau},$$
    where if $\tau > t$  we set $\widetilde r_{t,\tau}= 0$ and if $ \tau \le t$ we have that  $\widetilde r_{t,\tau}$ is given by the recurrence:
    $$
    \widetilde r_{t,\tau} = \sum_{i=1}^k \vgamma_i(a_t^*) \widetilde r_{t-i,\tau} + \delta_{t,\tau}\qquad \text{where} \qquad
     \delta_{t,\tau}:=
    \begin{cases}
    1\qquad t=\tau\\
    0\qquad t\neq \tau
    \end{cases}.
    $$
 	This way, we decompose the exogenous term $\rho_\tau$ as a linear combination of unitary impulses. Then by Assumption~\ref{ass:monotonicity} and~\ref{ass:stability}, recalling that  $\widetilde r_{t,\tau}= 0$ if $\tau > t$ and that $\widetilde r_{\tau,\tau}=1$, we have that for every $t > \tau$ it holds that:
    \begin{equation*}
    \widetilde r_{t,\tau}
     \le  \Gamma  \max_{i \in \dsb{k} } \widetilde r_{t-i,\tau}
     \le  \Gamma^2  \max_{i \in \dsb{k}}\max_{j \in \dsb{k}} \widetilde r_{t-i-j,\tau}
     \le \dots \le \Gamma^{\lceil (t-\tau)/k \rceil},
    \end{equation*}
    since we will encounter the $1=\delta_{\tau,\tau}$ after $\lceil (t-\tau)/k \rceil$ steps of unfolding.
    
    Now, we can manipulate this formula to have an expression of the full regret:
     \begin{align*}
        \sum_{t=1}^T r_t & \le \sum_{t=1}^T \left( \rho_t  +
        \sum_{\tau=1}^{t-1}\Gamma^{\lceil (t-\tau)/k \rceil} \rho_\tau  \right)\\
        &= \sum_{\tau=1}^T \left( 1 +
        \rho_\tau
        \sum_{t=\tau+1}^T\Gamma^{\lceil (t-\tau)/k \rceil} \right)\\
        &\stackrel{(a)}{\le}
        \sum_{\tau=1}^T \rho_\tau \left( 1 + 
        \sum_{s=1}^{+\infty} \Gamma^{\lceil s/k \rceil} \right) \\
        & \stackrel{(b)}{=}
        \sum_{\tau=1}^T \rho_\tau \left( 1 + 
        \sum_{l=1}^{+\infty} k \Gamma^{l} \right) \\
        & = \left(1 + \frac{\Gamma k}{1-\Gamma} \right) \sum_{\tau=1}^T \rho_\tau,
    \end{align*}
    where (a) follows from bounding the summation with the series and changing the index $s=t-\tau$ and (b) is obtained by observing that the exponent $\lceil s/k \rceil$ changes only when $s$ is divisible by $k$.
\end{proof}

\textbf{Counterexample to show that this bound is tight.}

There are $k$ arms:

$$\vgamma(a_1):=[\Gamma,0\dots 0],\qquad \vgamma(a_2):=[0,\Gamma,0\dots 0], \ \ \dots \ \vgamma(a_k):=[0,\dots 0, \Gamma].$$

All these arms have non-negative coefficients whose sum is bounded by $\Gamma$. If the sequence of internal regrets is:
$$\rho_t=
\begin{cases}
1\qquad &t=1\\
0\qquad &t>1
\end{cases},$$

and the sequence of arms is $a^*_1=1$, and $a^*_t=a_{t-1\ (mod\ k)+1}$ (which means $a_1, a_2, \dots , a_k, a_1, a_2, \dots$), we have:

$$r_1 = 1, r_2=\Gamma, \ r_3 = \Gamma, \ \dots , \ r_{k+1}=\Gamma,$$

and then, we start again with the same sequence of arms:

$$r_{k+2} = \Gamma^2, \ r_{k+3}=\Gamma^2, \ \dots , \ r_{2k+1}=\Gamma^2.$$

Making the sum of these terms for $t$ from one to infinity, we get:

$$\sum_{t=1}^\infty r_t= 1+k\sum_{t=1}^\infty \Gamma^t=1+\frac{k\Gamma}{1-\Gamma},$$

which is exactly the bound we get.

\begin{restatable}[]{lemma}{SuccessionBound} \label{lem:bounded_sequences}
Let $(\zs_t)_{t \in \dsb{T}}$ be the sequence of observation vectors observed by executing the learner's policy. If $\zs_0=(1,0,\dots,0)^T$, then, for every $\delta \in (0,1)$, with probability at least $1-\delta$, simultaneously for all $t \in \dsb{T}$, it holds that:
\begin{equation*}
\| \zs_{t-1} \|_2 \le \sqrt{1 + k \left( \frac{ m+ \eta }{1-\Gamma} \right)^2 },
\end{equation*}
where $\eta=\sqrt{2\sigma^2 \log(T/\delta)}$.
\end{restatable}

\begin{proof}
    Let $(\xi_t)_{t \in \dsb{T}}$ be the sequence of noises. We consider the event $\mathcal{E}=\bigcap_{t=1}^T \big \{|\xi_t| \le \eta \big \}$ prescribing that all noises are smaller than $\eta$ in absolute value.  
    By union bound, knowing that all the noises are independent $\sigma^2$-subgaussian random variables we, can bound the probability of event $\mathcal{E}$:
    $$
    \mathbb{P}(\mathcal{E}) = \mathbb{P}\left(\bigcap_{t=1}^T \big \{|\xi_t| \le \eta \big \} \right) \ge   1 - T e^{-\frac{\eta ^2}{2\sigma^2} } = 1- \delta,
    $$
    having set $\eta = \sqrt{2\sigma^2 \log(T/\delta)}$. Under event $\mathcal{E}$ and when $\zs_0 = (1,0,\dots,0)^T$, we prove by induction that all rewards $x_t$ are bounded in absolute value by $\frac{m+\eta}{1-\Gamma}$, regardless the actions played. For $T=1$, the statement is trivial since $x_1 = \gamma_0(a_1) + \eta_1$ and, thus, $|x_1| \le \gamma_0(a_1) + |\eta_1| \le m + \eta \le \frac{m+\eta}{1-\Gamma}$. Suppose the statement holds for all $s  < t$, we prove it for $t$. We have:
    \begin{align*}
    	 x_t = \gamma_0(a_t) + \sum_{i=1}^k \gamma_i(a_t) x_{t-i} + \eta_t  \quad \implies
    	 |x_t| & \le \gamma_0(a_t) + \sum_{i=1}^k \gamma_i(a_t) |x_{t-i}| + |\eta_t| \\
    	& \le m + \Gamma \frac{m+\Gamma}{1-\Gamma} + \eta = \frac{m+\eta}{1-\Gamma},
    \end{align*}
    where the first inequality uses Assumption~\ref{ass:monotonicity}, the second inequality follows from the inductive hypothesis and by Assumptions~\ref{ass:stability} and~\ref{ass:boundedness}. Passing to the observations vector, we have:
    \begin{align*}
    \left\| \zs_{t-1} \right\|_2^2 = 1 + \sum_{i=1}^k x_{t-i}^2 \le 1 + k \left(\frac{m+\eta}{1-\Gamma}\right)^2.
    \end{align*}
\end{proof}

For deriving the regret bound, we make use of the following result, known as \emph{Elliptic Potential Lemma}~\citep[][Lemma 19.4]{lattimore2020bandit}.
\begin{lemma}[Elliptic Potential Lemma]\label{lem:ell_pot}
Let $\Vs_0 \in \mathbb R^{d\times d}$ be a positive definite matrix and let $\bold a_1, \ldots , \bold a_n \in \mathbb R^{d}$
be a sequence of vectors such that $\|\bold a_t\|_2 \le L < +\infty$ for all $t \in \dsb{n}$. Let $\Vs_t = \Vs_0 + \sum_{s=1}^t \bold a_{s}\bold a_s^T$, Then:
$$\sum_{t=1}^n \min \{1, \|\bold a_{s}\|_{{\Vs_{t-1}}^{-1}}\}\le 2d\log \bigg(\frac{\mathrm{tr}(\Vs_0)+nL^2}{d\det(\Vs_0)^{1/d}}\bigg).$$
\end{lemma}

\RegretBound*

\begin{proof}
    We denote with $(x_t^*)_{t \in \dsb{T}}$ the sequence of rewards generated by playing the optimal policy and with $(x_t)_{t \in \dsb{T}}$ the sequence of rewards generated by playing \algnameshort.
Thanks to Lemma~\ref{lem:internal_regret_bound}, we have to bound the external regret only. Let $\delta \in (0,1)$, and define, as in the main paper, for every round $t \in \dsb{T}$ and action $a \in \As$:
\begin{equation*} 
\beta_t(a):=\sqrt{\lambda (m^2+1)} + \sigma \sqrt{2  \log\left( \frac{n}{\delta} \right) + \log \left( \frac{\det \Vs_t(a)}{\lambda^{k+1}} \right) }.
\end{equation*}
Let us define the confidence set $\mathcal{C}_t(a) \coloneqq \{\vgamma \in \mathbb{R}^{k+1} : \| \vgamma - \widehat{\vgamma}_{t-1}(a) \|_{\Vs_{t-1}(a)} \le \beta_{t-1}(a)\} $ and the optimistic estimate of the true parameter vector $\vgamma(a)$:
$$\widetilde{\vgamma}_t(a) \in \argmax_{\vgamma \in \mathcal{C}_t(a)} \inner{\vgamma}{\zs_{t-1}},$$
By Theorem~\ref{thr:concentration}, we have that, for every action $a \in \mathcal{A}$ and round $t \in\dsb{T}$, the true parameter vector satisfies $\vgamma(a) \in \mathcal{C}_t(a)$ with probability at least $1-\delta$. Therefore, with the same probability, we have:
\begin{align*}
    \inner{\vgamma(a_t^*) - \vgamma({a}_t)}{\zs_{t-1}} & = \underbrace{\inner{\vgamma(a_t^*) - \widetilde{\vgamma}_t({a}_t)}{\zs_{t-1}}}_{\le 0} + \inner{\widetilde{\vgamma}_t({a}_t) - \vgamma({a}_t)}{\zs_{t-1}} \\
    & \le  \inner{\widetilde{\vgamma}_t({a}_t) - \widehat{\vgamma}_{t-1}({a}_t)}{\zs_{t-1}} +  \inner{\widehat{\vgamma}_{t-1}({a}_{t}) - \vgamma({a}_t)}{\zs_{t-1}} \\
    & \le 2 \beta_{t-1}(a_t) \| \zs_{t-1} \|_{\Vs_{t-1}(a)^{-1}},
\end{align*}
where the first inequality follows from the optimism and in the last passage we have used Cauchy-Schwartz inequality, recalling that for every couple of vectors $\bold v, \bold w$ it holds $\inner{\bold v}{\bold w}\le \|\bold v\|_{\Vs_{t-1}(a)}\|\bold w\|_{\Vs_{t-1}(a)^{-1}}$, and having observed that $\vgamma({a}_t), \widetilde{\vgamma}_t({a}_t) \in \mathcal{C}_t({a}_t)$.

Furthermore, we observe that the external regret $\rho_t=\inner{\vgamma(a_t^*) - \vgamma({a}_t)}{\zs_{t-1}} \le \| \bm{z}_{t-1} \|_2+m $, since the coefficients $\gamma_j$ for $j\neq 0$ have a sum bounded by $\Gamma < 1$ and get multiplied by $\zs_{t-1}$, while $\gamma_0$, which is bounded by $m$ gets multiplied by $1$, then we have $\rho_t \le L+m=\mathcal{O}(m)$. By Lemma~\ref{lem:bounded_sequences} with probability of at least $1-\delta$ we have:
$$\| \bm{z}_t \|_2\le \sqrt{1 + k \left( \frac{ m+ \eta }{1-\Gamma} \right)^2 } \eqqcolon L,$$
where $\eta=\sqrt{2\sigma^2 \log(T/\delta)}$ and, consequently:
\begin{equation*}
    \rho_t \le m+L \eqqcolon C_1.
\end{equation*}
At this point, we proceed as follows:
\begin{align*}
    \rho_t \le 2\min\{C_1, \beta_{t-1}(a_t) \| \zs_{t-1} \|_{\Vs_{t-1}(a_t)^{-1}}\} \le  2 \max \{C_1, \beta_{t-1}(a_t)\} \min\{1, \| \zs_{t-1} \|_{\Vs_{t-1}(a_t)^{-1}}\}.
\end{align*}
Summing over $t \in \dsb{T}$, we obtain a bound on the cumulative external regret:
\begin{align*}
    \varrho(\text{\algnameshort},T) = \sum_{t=1}^T \rho_t &= \sum_{t=1}^T 1\cdot \rho_t \\
    & \le \sqrt{T \sum_{t=1}^T \rho_t^2 } \\
    &\le 2  \max \{C_1,\beta_{T-1}\} \sqrt{T  \sum_{t=1}^T \min\{1,  \| \zs_{t-1} \|_{\Vs_{t-1}(a_t)^{-1}}^2\}}
\end{align*}
where:
\begin{align*}
    \beta_{T-1} \coloneqq \max_{a\in \mathcal{A}}\beta_{T-1}(a),
\end{align*}
where the first inequality follows from an application of Cauchy-Schwartz inequality and the last passage holds since the sequence $\beta_{t}(a_t)$ is non-decreasing, and so we can bound each of them with their value at $t=T$. Now, we are finally able to use the \textit{Elliptic Potential Lemma} (Lemma~\ref{lem:ell_pot}):
\begin{align*}
     \sum_{t=1}^T \min\{1,  \| \zs_{t-1} \|_{\Vs_{t-1}(a_t)^{-1}}^2\} & = \sum_{a \in \As} \sum_{l \in \mathcal{O}_T(a)} \min\{1,  \| \zs_{l-1} \|_{\Vs_{l-1}(a)^{-1}}^2\}  \\
     & \le \sum_{a \in \As} 2(k+1)\log \bigg( \frac{\lambda(k+1) + |\mathcal{O}_T(a)| L^2}{\lambda(k+1)}\bigg) \\
     & \le 2n(k+1) \log \bigg(1 + \frac{T L^2}{n \lambda(k+1)}\bigg),
\end{align*}
where the first inequality follows from an application of the elliptic potential lemma for each action $a \in \As$ observing that $\Vs_0=\lambda \mathbf{I}_{k+1}$ and, consequently, $\text{tr}(\Vs_0) = \lambda(k+1)$ and $\det(\Vs_0)^{1/(k+1)} = \lambda$. The second inequality follows by observing that $\sum_{a \in \As}  |\mathcal{O}_T(a)| = T $ and since the $\log$ is a concave function, the worst allocation of pulls is the uniform one.
Now that we have bounded the inner summation, we can state that:
$$\varrho(\text{\algnameshort},T) = \sum_{t=1}^T \rho_t \le 2  \max \{C_1,\beta_{T-1}\} \sqrt{2 Tn(k+1)\log \bigg(1 + \frac{TL^2}{n \lambda(k+1)}\bigg)}.$$
To conclude, we bound the term $\beta_{T-1}$ as follows:
\begin{align*}\beta_{T-1}&=\sqrt{\lambda (m^2+1)} + \sigma \max_{a\in \mathcal{A}} \sqrt{2  \log\left( \frac{n}{\delta} \right) + \log \left( \frac{\det \Vs_{T-1}(a)}{\lambda^{k+1}} \right) }\\
&\le \sqrt{\lambda (m^2+1)} + \sigma \sqrt{2  \log\left( \frac{n}{\delta} \right) +  (k+1)\log \bigg(\frac{\lambda (k+1)+TL^2}{\lambda(k+1)}\bigg)}.\end{align*}
Therefore, by highlighting the dependences on $m$, $k$, $\sigma$, and $\Gamma$, we have:
$$\beta_{T-1}=\widetilde O\left(m+  \sigma \sqrt{k+1}\right), \qquad C_1 = \widetilde{\mathcal{O}}\left( 1 + \sqrt{k} \frac{m+\sigma}{1-\Gamma}\right).$$
These results hold with probability $1-2\delta$. We set $\delta=(2T)^{-1}$. Putting all together, we obtain:
\begin{align*}
	\varrho(\text{\algnameshort},T) =  \sum_{t=1}^T \rho_t \le \widetilde{\mathcal{O}}\left(\frac{(m + \sigma)\sqrt{n(k+1)T}}{1-\Gamma}\right),
\end{align*}
and, applying the previous Lemma \ref{lem:internal_regret_bound}, this results in:
\begin{align*}
R(\text{\algnameshort},T)\le \widetilde{\mathcal{O}}\bigg(\frac{(m + \sigma)(k+1)^{3/2}\sqrt{nT}}{(1-\Gamma)^2}\bigg).
\end{align*}
\end{proof}

\section{OPTIMAL POLICY WITHOUT NOISE}
\label{apx:optimal_policy_no_noise}

In the case of no noise, our system writes:

\begin{equation}\label{eq:no_noise}
x_t = \gamma_0(a_t) + \sum_{i=1}^{k}\gamma_i(a_t) x_{t-i}. 
\end{equation}

In this case, the process evolution is deterministic. Therefore, even if it is still true that the optimal policy is given by Theorem~\ref{thr:optimal_policy}, it is possible to say that there is a constant policy that is asymptotically optimal, in the sense that its cumulative regret is bounded by a constant.
This policy is given by:
\begin{equation}\label{eq:opt_const}
a^* \in  \argmax_{a\in \mathcal A} \frac{\gamma_0(a_t)}{1-\sum_{i=1}^{k}\gamma_i(a_t)}.\end{equation}

This result is not surprising. In fact, this action makes the process converge to the highest possible stationary reward, which is of course $\argmax_{a\in \mathcal A} \frac{\gamma_0(a_t)}{1-\sum_{i=1}^{k}\gamma_i(a_t)}$. Precisely, the following result holds.

\begin{thr}
    Let us consider the problem formulation of Equation~\eqref{eq:no_noise}. Define:
    \begin{equation*}
a^*= \argmax_{a\in \mathcal A} \frac{\gamma_0(a_t)}{1-\sum_{i=1}^{k}\gamma_i(a_t)},\end{equation*}
    as in Equation~\eqref{eq:opt_const}. Then, there exist no policy $\bm{\pi}$ (even non-constant) such that:
        $$\limsup_{t\to +\infty} x_t^{\bm{\pi}}-x_t^*>0$$
        (where $x_t^{\bm{\pi}}$ denotes the sequence obtained with policy $\bm{\pi}$, while $x_t^*$ is the one relative to $a^*$). Moreover, the cumulative regret with respect to the actual optimal policy is bounded by:
        $$\gamma_0(a^*)\frac{k}{(1-\Gamma)^2}.$$
\end{thr}
\begin{proof}
    If we play always $a^*$, we have:
    $$\limsup_{t\to +\infty} x_t^*=
    \frac{\gamma_0(a^*)}{1-\sum_{i=1}^{k}\gamma_i(a^*)},$$
    by imposing the condition of stationarity. For the rest of the proof, let us denote:
    $$x^*:=\frac{\gamma_0(a^*)}{1-\sum_{i=1}^{k}\gamma_i(a^*)}.$$
    
    Now, we prove that, for any policy $\bm{\pi}$, we cannot achieve an $x_t>x^*$. By contradiction, if $\limsup_{t\to \infty} x_t^{\bm{\pi}}-x_t^*>0$, then the set $\{t \in \mathbb{N} \,:\ x_t>x^*\}$ is non-empty. Let $t_0 = \min \{t \in \mathbb{N} \ :\ x_t>x^*\}$. Then, by definition:
    
    $$x_{t_0} = \gamma_0(a_{t_0}) + \sum_{i=1}^{k}\gamma_i(a_{t_0}) x_{{t_0}-i}.$$
    
    Recalling that $t_0$ is the first time in which we surpass $x^*$, we have:
    \begin{align*}
        x^*<x_{t_0} &= \gamma_0(a_{t_0}) + \sum_{i=1}^{k}\gamma_i(a_{t_0}) x_{{t_0}-i} \le \gamma_0(a_{t_0}) + \sum_{i=1}^{k}\gamma_i(a_{t_0})x^*.
    \end{align*}
    This inequality entails that:
    $$\Big (1-\sum_{i=1}^{k}\gamma_i(a_{t_0})\Big )x^*<\gamma_0(a_{t_0}),$$
    and, therefore:
    $$\frac{\gamma_0(a^*)}{1-\sum_{i=1}^{k}\gamma_i(a^*)}=x^*<\frac{\gamma_0(a_{t_0})}{1-\sum_{i=1}^{k}\gamma_i(a_{t_0})},$$
    which contradicts the definition of $a^*$.
    
    For the second part, we start considering that the regret obtained by using constant action $a^*$ is bounded by:
    $$\sum_{t=1}^{+\infty} x^*-x_t,$$
    since $x^*$ is the maximum instantaneous reward that every policy can achieve. Now, note that $\gamma_0(a^*)>0$, otherwise it could not be the optimal action. At this point, we have for $0<t\le k$ that $x_{t}\ge \gamma_0(a^*)$, 
    by simply using the fact that all the coefficients of the autoregressive model are non-negative. From this fact we have  for $k<t\le 2k$ that $x_{t}\ge \gamma_0(a^*)(1 + \sum_{i=1}^k \gamma_i(a^*))$; and generalizing:
    $$\forall j>0 \qquad \text{and} \qquad jk-k<t\le jk: \qquad x_{t}\ge \gamma_0(a^*)\Big( \sum_{\ell=0}^j (\Gamma^*)^\ell \Big),\qquad \Gamma^* = \sum_{i=1}^k \gamma_i(a^*).$$
Therefore, we have $x_t\ge \gamma_0(a^*)\frac{1-\Gamma^{\lfloor t/k\rfloor}}{1-\Gamma}$, which means:
    \begin{align*}
        R_t &\le \sum_{t=1}^{+\infty} x^*-x_t \\
        & \le \sum_{t=1}^{+\infty} x^*-\gamma_0(a^*)\frac{1-\Gamma^{\lfloor t/k\rfloor}}{1-\Gamma}\\
        & = \gamma_0(a^*)\sum_{t=1}^{+\infty} \frac{1}{1-\Gamma}-\frac{1-\Gamma^{\lfloor t/k\rfloor}}{1-\Gamma}\\
        & = \gamma_0(a^*)\sum_{t=1}^{+\infty} \frac{\Gamma^{\lfloor t/k\rfloor}}{1-\Gamma}\\
        & = \gamma_0(a^*)\frac{k}{(1-\Gamma)^2}.
    \end{align*}
\end{proof}

\section{DISCUSSION ON ASSUMPTION~\ref{ass:monotonicity}}
\label{apx:assumption1a}
In this appendix, we further detail the meaning of Assumption 1.a related to non-negative coefficients governing the AR process (Assumption~\ref{ass:monotonicity}). Even if, theoretically, this setting is less general than the one that considers all possible values for the parameters, we argue that, for the real-world applications of interest, considering negative coefficients is not meaningful. \\
Before introducing our example, let us remark on the meaning of $x_t$ in practice. This value represents the sales volume in the case of pricing, the value of a stock in the stock market, the number of customers that an e-commerce website may have, and so on. In all these real-world scenarios, the quantity $x_t$ is meaningful whenever we consider non-negative values that we want to maximize. We argue that when Assumption~\ref{ass:monotonicity} is not fulfilled (\ie at least one $\gamma_i(a)$ is negative), the positivity of $x_t$ is no longer ensured. \\
Consider now the example presented in Figure~\ref{fig:example_ass_bounces}, where we present a general scenario in which, at time $\tau$, we are in a given with a certain positive $x_\tau$. Consider, for the sake of simplicity, a noiseless setting with $k=1$ (\ie an AR(1) process) and, for a given action $i$, we have $\gamma_0(a) = 0$. Consider now $\gamma_1(a) < 0$. Figure~\ref{fig:example_ass_bounces} shows what will happen in this case. The value of $x_t$ continuously changes its sign at each time step, and this behavior is not compatible with the real-world phenomena of our interest. This is even more unrealistic if we think about the scenario in which we have another value of the state $\overline{x}_\tau > x_\tau$. In this scenario, after performing the same action $i$, we will observe that the best-starting state $\overline{x}_\tau$ leads to a worst next state $\overline{x}_{\tau + 1} < x_{\tau + 1}$. This behavior has no practical meaning in the applications of our interest. Given these considerations, we can derive that the worst possible effect of a given action is to \emph{reset} the state, which corresponds to have $\gamma_1(a)=0$. A representation of this phenomenon is drawn in Figure~\ref{fig:example_ass_non_neg}. From this figure, it is possible to notice how a process can always decrease as an effect of an action, even for $\gamma_1(a)>0$. \\
This consideration trivially generalizes for any $k>1$ given a generic state representation $\zs_\tau$.

\begin{figure}[t!]
\begin{minipage}{.47\textwidth}
    \tikzset{every picture/.style={line width=0.75pt}} 

\begin{tikzpicture}[x=0.75pt,y=0.75pt,yscale=-1,xscale=1]
\draw    (20,240) -- (20,240) ;
\draw    (90.25,300) -- (90.25,143.25) ;
\draw [shift={(90.25,140.25)}, rotate = 90] [fill={rgb, 255:red, 0; green, 0; blue, 0 }  ][line width=0.08]  [draw opacity=0] (6.25,-3) -- (0,0) -- (6.25,3) -- cycle    ;
\draw    (60,240.25) -- (236.4,240.2) ;
\draw [shift={(239.4,240.2)}, rotate = 179.98] [fill={rgb, 255:red, 0; green, 0; blue, 0 }  ][line width=0.08]  [draw opacity=0] (6.25,-3) -- (0,0) -- (6.25,3) -- cycle    ;
\draw    (120.25,237.5) -- (120.35,243.95) ;
\draw    (170.25,237.5) -- (170.35,243.95) ;
\draw    (220.25,237.25) -- (220.35,243.7) ;
\draw   (118.75,179.88) .. controls (118.75,179.12) and (119.37,178.5) .. (120.13,178.5) .. controls (120.88,178.5) and (121.5,179.12) .. (121.5,179.88) .. controls (121.5,180.63) and (120.88,181.25) .. (120.13,181.25) .. controls (119.37,181.25) and (118.75,180.63) .. (118.75,179.88) -- cycle ;
\draw   (169.25,270.13) .. controls (169.25,269.37) and (169.87,268.75) .. (170.63,268.75) .. controls (171.38,268.75) and (172,269.37) .. (172,270.13) .. controls (172,270.88) and (171.38,271.5) .. (170.63,271.5) .. controls (169.87,271.5) and (169.25,270.88) .. (169.25,270.13) -- cycle ;
\draw   (219,225.88) .. controls (219,225.12) and (219.62,224.5) .. (220.38,224.5) .. controls (221.13,224.5) and (221.75,225.12) .. (221.75,225.88) .. controls (221.75,226.63) and (221.13,227.25) .. (220.38,227.25) .. controls (219.62,227.25) and (219,226.63) .. (219,225.88) -- cycle ;
\draw  [color={rgb, 255:red, 0; green, 0; blue, 0 }  ,draw opacity=1 ] (118.75,140.63) .. controls (118.75,139.87) and (119.37,139.25) .. (120.13,139.25) .. controls (120.88,139.25) and (121.5,139.87) .. (121.5,140.63) .. controls (121.5,141.38) and (120.88,142) .. (120.13,142) .. controls (119.37,142) and (118.75,141.38) .. (118.75,140.63) -- cycle ;
\draw  [color={rgb, 255:red, 0; green, 0; blue, 0 }  ,draw opacity=1 ] (218.63,215.25) .. controls (218.63,214.49) and (219.24,213.88) .. (220,213.88) .. controls (220.76,213.88) and (221.38,214.49) .. (221.38,215.25) .. controls (221.38,216.01) and (220.76,216.63) .. (220,216.63) .. controls (219.24,216.63) and (218.63,216.01) .. (218.63,215.25) -- cycle ;
\draw  [color={rgb, 255:red, 0; green, 0; blue, 0 }  ,draw opacity=1 ] (169.25,290.13) .. controls (169.25,289.37) and (169.87,288.75) .. (170.63,288.75) .. controls (171.38,288.75) and (172,289.37) .. (172,290.13) .. controls (172,290.88) and (171.38,291.5) .. (170.63,291.5) .. controls (169.87,291.5) and (169.25,290.88) .. (169.25,290.13) -- cycle ;
\draw [color={rgb, 255:red, 155; green, 155; blue, 155 }  ,draw opacity=1 ] [dash pattern={on 4.5pt off 4.5pt}]  (120.13,179.88) -- (170.63,270.13) ;
\draw [color={rgb, 255:red, 155; green, 155; blue, 155 }  ,draw opacity=1 ] [dash pattern={on 4.5pt off 4.5pt}]  (220.38,225.88) -- (170.63,270.13) ;
\draw [color={rgb, 255:red, 155; green, 155; blue, 155 }  ,draw opacity=1 ] [dash pattern={on 4.5pt off 4.5pt}]  (120.13,140.63) -- (170.63,290.13) ;
\draw [color={rgb, 255:red, 155; green, 155; blue, 155 }  ,draw opacity=1 ] [dash pattern={on 4.5pt off 4.5pt}]  (220,215.25) -- (170.63,290.13) ;

\draw (116,245) node [anchor=north west][inner sep=0.75pt]  [font=\scriptsize] [align=left] {$\tau$};
\draw (157,244) node [anchor=north west][inner sep=0.75pt]  [font=\scriptsize] [align=left] {$\tau+1$};
\draw (206,244) node [anchor=north west][inner sep=0.75pt]  [font=\scriptsize] [align=left] {$\tau+2$};

\draw (240,240) node [anchor=north west][inner sep=0.75pt]  [font=\scriptsize] [align=left] {$t$};

\draw (73,140) node [anchor=north west][inner sep=0.75pt]  [font=\scriptsize] [align=left] {$x_t$};

\draw (122,176) node [anchor=north west][inner sep=0.75pt]  [font=\scriptsize] [align=left] {$x_\tau$};

\draw (173,267) node [anchor=north west][inner sep=0.75pt]  [font=\scriptsize] [align=left] {$x_{\tau+1}$};

\draw (223,222) node [anchor=north west][inner sep=0.75pt]  [font=\scriptsize] [align=left] {$x_{\tau+2}$};

\draw (122,136) node [anchor=north west][inner sep=0.75pt]  [font=\scriptsize] [align=left] {$\overline{x}_\tau$};

\draw (173,286) node [anchor=north west][inner sep=0.75pt]  [font=\scriptsize] [align=left] {$\overline{x}_{\tau+1}$};

\draw (223,209) node [anchor=north west][inner sep=0.75pt]  [font=\scriptsize] [align=left] {$\overline{x}_{\tau+2}$};

\end{tikzpicture}
    \captionof{figure}{An illustration of the effect of a negative $\gamma_1(a)$ over time.\\}
    \label{fig:example_ass_bounces}
\end{minipage}
\hfill
\begin{minipage}{.47\textwidth}
    \tikzset{every picture/.style={line width=0.75pt}} 

\begin{tikzpicture}[x=0.75pt,y=0.75pt,yscale=-1,xscale=1]

\draw    (220,240) -- (220,240) ;
\draw    (290,300) -- (290,143.5) ;
\draw [shift={(290,140.5)}, rotate = 90] [fill={rgb, 255:red, 0; green, 0; blue, 0 }  ][line width=0.08]  [draw opacity=0] (6.25,-3) -- (0,0) -- (6.25,3) -- cycle    ;
\draw    (260,240.5) -- (436.75,240.25) ;
\draw [shift={(439.75,240.25)}, rotate = 179.92] [fill={rgb, 255:red, 0; green, 0; blue, 0 }  ][line width=0.08]  [draw opacity=0] (6.25,-3) -- (0,0) -- (6.25,3) -- cycle    ;
\draw    (320.5,238.25) -- (320.6,244.7) ;
\draw    (370,236.75) -- (370.1,243.2) ;
\draw   (319,180.63) .. controls (319,179.87) and (319.62,179.25) .. (320.38,179.25) .. controls (321.13,179.25) and (321.75,179.87) .. (321.75,180.63) .. controls (321.75,181.38) and (321.13,182) .. (320.38,182) .. controls (319.62,182) and (319,181.38) .. (319,180.63) -- cycle ;
\draw  [color={rgb, 255:red, 208; green, 2; blue, 27 }  ,draw opacity=1 ] (369,259.25) .. controls (369,258.49) and (369.62,257.88) .. (370.38,257.88) .. controls (371.13,257.88) and (371.75,258.49) .. (371.75,259.25) .. controls (371.75,260.01) and (371.13,260.63) .. (370.38,260.63) .. controls (369.62,260.63) and (369,260.01) .. (369,259.25) -- cycle ;
\draw [color={rgb, 255:red, 155; green, 155; blue, 155 }  ,draw opacity=1 ] [dash pattern={on 4.5pt off 4.5pt}]  (320.38,180.63) -- (370.38,182.63) ;
\draw [color={rgb, 255:red, 155; green, 155; blue, 155 }  ,draw opacity=1 ] [dash pattern={on 4.5pt off 4.5pt}]  (370.13,240.38) -- (320.38,180.63) ;
\draw [color={rgb, 255:red, 255; green, 140; blue, 140 }  ,draw opacity=1 ] [dash pattern={on 4.5pt off 4.5pt}]  (320.38,180.63) -- (370.38,260.63) ;
\draw   (369,182.63) .. controls (369,181.87) and (369.62,181.25) .. (370.38,181.25) .. controls (371.13,181.25) and (371.75,181.87) .. (371.75,182.63) .. controls (371.75,183.38) and (371.13,184) .. (370.38,184) .. controls (369.62,184) and (369,183.38) .. (369,182.63) -- cycle ;
\draw   (369,200.38) .. controls (369,199.62) and (369.62,199) .. (370.38,199) .. controls (371.13,199) and (371.75,199.62) .. (371.75,200.38) .. controls (371.75,201.13) and (371.13,201.75) .. (370.38,201.75) .. controls (369.62,201.75) and (369,201.13) .. (369,200.38) -- cycle ;
\draw   (368.75,240.38) .. controls (368.75,239.62) and (369.37,239) .. (370.13,239) .. controls (370.88,239) and (371.5,239.62) .. (371.5,240.38) .. controls (371.5,241.13) and (370.88,241.75) .. (370.13,241.75) .. controls (369.37,241.75) and (368.75,241.13) .. (368.75,240.38) -- cycle ;
\draw   (368.75,219.88) .. controls (368.75,219.12) and (369.37,218.5) .. (370.13,218.5) .. controls (370.88,218.5) and (371.5,219.12) .. (371.5,219.88) .. controls (371.5,220.63) and (370.88,221.25) .. (370.13,221.25) .. controls (369.37,221.25) and (368.75,220.63) .. (368.75,219.88) -- cycle ;
\draw [color={rgb, 255:red, 155; green, 155; blue, 155 }  ,draw opacity=1 ] [dash pattern={on 4.5pt off 4.5pt}]  (320.38,180.63) -- (370.38,200.38) ;
\draw [color={rgb, 255:red, 155; green, 155; blue, 155 }  ,draw opacity=1 ] [dash pattern={on 4.5pt off 4.5pt}]  (320.38,180.63) -- (370.13,219.88) ;
\draw [color={rgb, 255:red, 255; green, 140; blue, 140 }  ,draw opacity=1 ] [dash pattern={on 4.5pt off 4.5pt}]  (320.38,180.63) -- (370.13,280.38) ;
\draw  [color={rgb, 255:red, 208; green, 2; blue, 27 }  ,draw opacity=1 ] (368.75,280.38) .. controls (368.75,279.62) and (369.37,279) .. (370.13,279) .. controls (370.88,279) and (371.5,279.62) .. (371.5,280.38) .. controls (371.5,281.13) and (370.88,281.75) .. (370.13,281.75) .. controls (369.37,281.75) and (368.75,281.13) .. (368.75,280.38) -- cycle ;

\draw (317,245) node [anchor=north west][inner sep=0.75pt]  [font=\scriptsize] [align=left] {$\tau$};
\draw (356,243) node [anchor=north west][inner sep=0.75pt]  [font=\scriptsize] [align=left] {$\tau+1$};

\draw (305,176) node [anchor=north west][inner sep=0.75pt]  [font=\scriptsize] [align=left] {$x_\tau$};

\draw (275,138) node [anchor=north west][inner sep=0.75pt]  [font=\scriptsize] [align=left] {$x_t$};

\draw (436,243) node [anchor=north west][inner sep=0.75pt]  [font=\scriptsize] [align=left] {$t$};

\draw (389,195) node [anchor=north west][inner sep=0.75pt]  [font=\scriptsize] [align=left] {$0 < \gamma_1(i) < 1$};

\draw (387,221) node [anchor=north west][inner sep=0.75pt]  [font=\scriptsize] [align=left] {$\gamma_1(i) = 0 $};

\draw (387,263) node [anchor=north west][inner sep=0.75pt]  [font=\scriptsize] [align=left] {$\gamma_1(i) < 0$};

\draw   (374.17,221.71) .. controls (378.84,221.71) and (381.17,219.38) .. (381.17,214.71) -- (381.17,210.92) .. controls (381.17,204.25) and (383.5,200.92) .. (388.17,200.92) .. controls (383.5,200.92) and (381.17,197.59) .. (381.17,190.92)(381.17,193.92) -- (381.17,187.13) .. controls (381.17,182.46) and (378.84,180.13) .. (374.17,180.13) ;
\draw   (374.42,283.46) .. controls (378.15,283.45) and (380.01,281.57) .. (380,277.84) -- (380,277.84) .. controls (379.99,272.51) and (381.85,269.83) .. (385.58,269.82) .. controls (381.85,269.83) and (379.97,267.17) .. (379.95,261.83)(379.96,264.23) -- (379.95,261.83) .. controls (379.94,258.1) and (378.06,256.24) .. (374.33,256.25) ;
\draw    (387.63,230.32) -- (375.5,236.67) ;
\draw [shift={(372.84,238.06)}, rotate = 332.37] [fill={rgb, 255:red, 0; green, 0; blue, 0 }  ][line width=0.08]  [draw opacity=0] (3.57,-1.72) -- (0,0) -- (3.57,1.72) -- cycle    ;

\end{tikzpicture}
    \captionof{figure}{The effect of $\gamma_1(a)$ in the evolution of the state $x_t$, in the case of a non-negative one (in black), and a negative one (in red).}
    \label{fig:example_ass_non_neg}
\end{minipage}
\end{figure}

\section{ADDITIONAL EXPERIMENTAL RESULTS}
\label{apx:additionalexperiments}

In this appendix, we provide additional experimental results.
In Appendix~\ref{apx:standardbandit}, we assert the effectiveness of \algnameshort in the classic stochastic bandit problem by comparing its performances with two standard baselines from the literature.
In Appendix~\ref{apx:standardbandit_missk}, we stress the effect of misspecifying parameter $\overline{k}$ in the standard multi-armed bandit problem. 
Finally, in Appendix~\ref{apx:ar1}, we provide experimental results in the particular case of autoregressive processes of order 1 (\ie $k=1$).

\subsection{Stochastic Bandit Problem}
\label{apx:standardbandit}

\textbf{Setting}~~~We evaluate \algnameshort in the special case $k=0$. This problem is equivalent to solving a standard stochastic bandit problem. This experiment compares the performances of \algnameshort in this setting against well-known gold standards: \ucbone and \expthree. The competing algorithms are evaluated in terms of cumulative regret w.r.t.~the setting-specific clairvoyant. The three settings differ in the values of $m \in \{2, 7.5\}$ (\ie the maximum arms' expected reward) and the values of $\sigma \in \{0.9, 1.25, 2\}$, the noise's standard deviation.

\begin{figure*}[t!]
    \centering
    \subfloat[Setting \texttt{A}.]{\resizebox{0.32\linewidth}{!}{\includegraphics[]{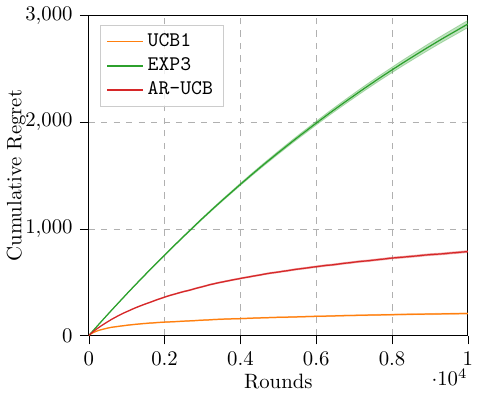}} \label{fig:ar0_1}}
    \subfloat[Setting \texttt{B}.]{\resizebox{0.32\linewidth}{!}{\includegraphics[]{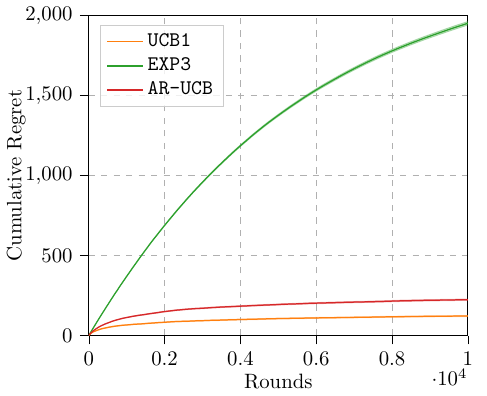}} \label{fig:ar0_2}}
    \subfloat[Setting \texttt{C}.]{\resizebox{0.32\linewidth}{!}{\includegraphics[]{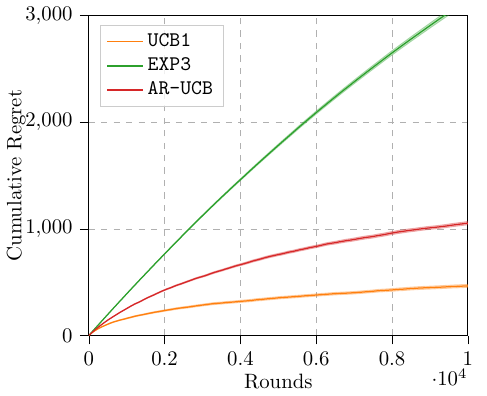}} \label{fig:ar0_3}}
    \caption{Cumulative regret of \algnameshort, \ucbone, and \expthree in the case of $k=0$ (100 runs, mean $\pm$ std).}
    \label{fig:cum_regrets_apx_k0}
\end{figure*}

\textbf{Results}~~~Figure~\ref{fig:cum_regrets_apx_k0} shows the average cumulative regrets for \algnameshort, \ucbone, and \expthree. We immediately observe that all the algorithms suffer sublinear cumulative regret, as expected since they are all able to provide no-regret theoretical guarantees in this setting. In all the experiments, \ucbone outperforms all the other algorithms since it is specifically designed for the scenario under analysis. \algnameshort, as expected, performs properly in this setting since, as already discussed in Section~\ref{sec:boundRegret}, its regret is asymptotically optimal when $k=0$.

\subsection{On the Misspecification of $k$ in Stochastic Bandit Problem}
\label{apx:standardbandit_missk}

\textbf{Setting}~~~We evaluate \algnameshort in the special case $k=0$. This problem is equivalent to solving a standard stochastic bandit problem. This experiment compares the performances of \algnameshort under different values of the parameter $\overline{k}$. In particular, this experiment aims to highlight the performances of \algnameshort under a misspecification of the process memory in the special case where the true underlying process does not present a dynamic temporal structure. The parameters $\vgamma_0(a)$ have been sampled by a uniform distribution having support $[6,7]$, and $m$ is set to $10$. The noise's standard deviation $\sigma$ is set to $1$. The number of actions is $n=7$.

\textbf{Results}~~~Figure~\ref{fig:apx_misspec_k} shows the average cumulative regrets for \algnameshort under different values of $\overline{k}$, when the true value is $k=0$. The figure shows that \algnameshort is capable of achieving sublinear cumulative regret even when the misspecification is severe (\eg $\overline{k}=16$), coherently with the theoretical results, the performance degrades as the misspecification grows.

\begin{figure}[t!]
    \centering
    \resizebox{0.33\linewidth}{!}{\includegraphics[]{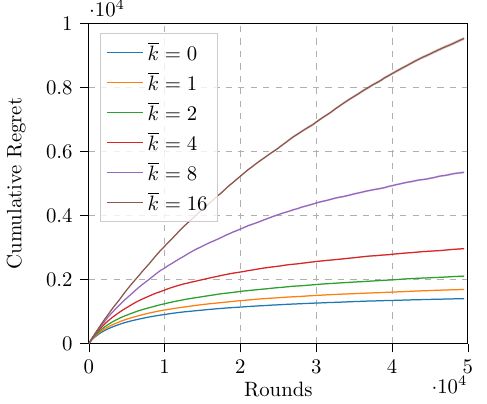}}
    \captionof{figure}{Cumulative regret of \algnameshort in the case of $k=0$, in with $\overline{k}$ parameter misspecified (100 runs, mean $\pm$ std).}
    \label{fig:apx_misspec_k}
\end{figure}

\subsection{AR(1) Bandit Problem}
\label{apx:ar1}
AR(1) processes are the simplest autoregressive processes. Therefore, we will present a specific analysis of this setting to show how \algnameshort and the baselines perform when the complexity given by the dynamic temporal structure is minimal. Results show how even the minimal autoregressive contribution can lead all the algorithms (except for \algnameshort) to linear cumulative regret.

\textbf{Setting}~~~We evaluate \algnameshort in the case $k=1$. This is the simplest setting in which an autoregressive component contributes to the reward. This experiment compares the performances of \algnameshort in this setting against the same baselines as Section \ref{subsec:exp_1}. The competing algorithms are evaluated in terms of cumulative regret w.r.t.~the setting-specific clairvoyant. The three settings differ in the values of $m \in \{2, 8, 10\}$ (\ie the maximum arms' expected reward) and the values of $\sigma \in \{1, 1.25, 2\}$, the noise's standard deviation. The values of the $\gamma_1(a)$ parameters have been sampled from uniform distributions having their sampling ranges inside $[0,1)$. The number of actions is $n=7$.

\begin{figure*}[t!]
    \centering
    \subfloat[Setting \texttt{A}.]{\resizebox{0.32\linewidth}{!}{\includegraphics[]{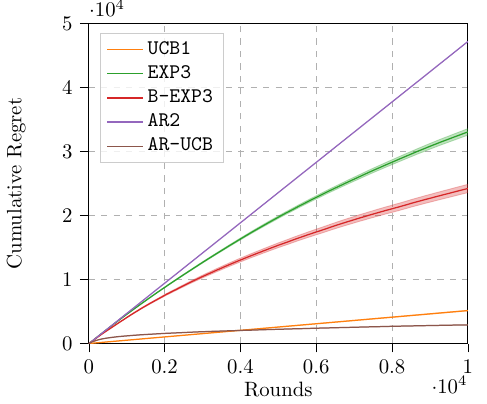}} \label{fig:ar1_1}}
    \subfloat[Setting \texttt{B}.]{\resizebox{0.32\linewidth}{!}{\includegraphics[]{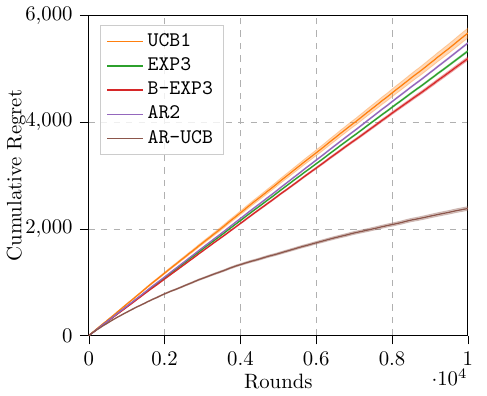}} \label{fig:ar1_2}}
    \subfloat[Setting \texttt{C}.]{\resizebox{0.32\linewidth}{!}{\includegraphics[]{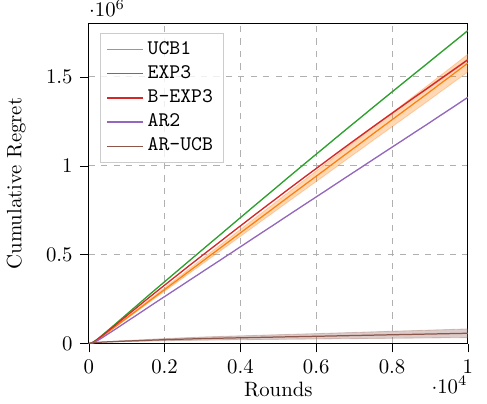}} \label{fig:ar1_3}}
    \caption{Cumulative regret of \algnameshort and the others bandit baselines in the case of $k=1$ (100 runs, mean $\pm$ std).}
    \label{fig:cum_regrets_apx_k1}
\end{figure*}

\textbf{Results}~~~Figure~\ref{fig:cum_regrets_apx_k1} shows the average cumulative regrets for all the competing algorithms. We immediately observe that the only algorithms able to achieve sublinear regret are \algnameshort (in all three settings), \batchexpthree (first and third experiments), and \expthree (first experiment only). Such a result is unsurprising since none of the baselines has specific theoretical guarantees in the Autoregressive Bandit problem, even in the simple scenario when $k=1$. Even though, we decided to adopt these algorithms as baselines since they represent the gold standard algorithms in the bandit literature (\ucbone, \expthree) and the algorithms that solve problems near to ours (\batchexpthree, \artwo), respectively.

\end{document}